\documentclass{article}

\usepackage[nonatbib,preprint]{classes/neurips_2023}

\usepackage[numbers]{natbib}

\usepackage[utf8]{inputenc} % allow utf-8 input
\usepackage[T1]{fontenc}    % use 8-bit T1 fonts
\usepackage[colorlinks, citecolor=blue]{hyperref}       % hyperlinks
\usepackage{url}            % simple URL typesetting
\usepackage{booktabs}       % professional-quality tables
\usepackage{amsfonts}       % blackboard math symbols
\usepackage{nicefrac}       % compact symbols for 1/2, etc.
\usepackage{microtype}      % microtypography
\usepackage{tabularx}
\usepackage{multirow}
\usepackage{multicol}
\usepackage[table,dvipsnames]{xcolor}

\usepackage{amsmath}
\usepackage{amssymb}
\usepackage{mathtools}
\usepackage{amsthm}
\usepackage{bbm} % we need this for \One
\usepackage{bm} % bold greek letters support

\RequirePackage{algorithm}
\RequirePackage{algorithmic}

\usepackage{pifont}% http://ctan.org/pkg/pifont
\newcommand{\cmark}{\ding{51}}%
\newcommand{\xmark}{\ding{55}}%

\usepackage[disable,textsize=tiny]{todonotes}
\newcommand{\del}[1]{}
\newcommand{\add}[1]{#1}

\theoremstyle{plain}
\newtheorem{theorem}{Theorem}[section]
\newtheorem{proposition}[theorem]{Proposition}
\newtheorem{lemma}[theorem]{Lemma}
\newtheorem{corollary}[theorem]{Corollary}
\theoremstyle{definition}
\newtheorem{definition}[theorem]{Definition}
\newtheorem{assumption}[theorem]{Assumption}
\theoremstyle{remark}
\newtheorem{remark}[theorem]{Remark}

\newcommand{\A}{\mathcal{A}}
\newcommand{\X}{\mathcal{X}}
\newcommand{\D}{\mathcal{D}}
\newcommand{\R}{\mathbb{R}}
\newcommand{\df}{\gamma}

\newcommand{\vvec}{{\bm{v}}}
\newcommand{\lvec}{\bm{\lambda}}
\newcommand{\tvec}{\bm{\theta}}

\newcommand{\Psim}{\bm{\Psi}}
\newcommand{\Phim}{\bm{\Phi}}
\newcommand{\nuvec}{\bm{\nu}}

\newcommand{\n}{n}

\newcommand{\suma}{\sum_{a}}

\newcommand{\suminfty}{\sum_{t=0}^\infty}
\newcommand{\One}[1]{\mathbbm{1}\{ #1 \}}

\providecommand\given{}
\newcommand\SetSymbol[1][]{%
    \nonscript\:#1\vert
    \allowbreak
    \nonscript\:
    \mathopen{}}

\DeclarePairedDelimiterX\ip[2]{\langle}{\rangle}{#1,#2}

\let\P\undefined
\DeclarePairedDelimiterXPP\P[1]{\mathbb{P}}(){}{
    \renewcommand\given{\SetSymbol[\delimsize]}
    #1
}

\DeclarePairedDelimiterXPP\E[1]{\mathbb{E}}[]{}{
    \renewcommand\given{\SetSymbol[\delimsize]}
    #1
}

\DeclarePairedDelimiterXPP\Es[2]{\mathbb{E}_{#1}}[]{}{
    \renewcommand\given{\SetSymbol[\delimsize]}
    #2
}

\DeclareMathOperator{\Tr}{Tr}

\newcommand{\feat}[1][x,a]{\bm{\varphi}({#1})}
\newcommand{\phit}{\bm{\varphi}_{t}}
\newcommand{\phiti}{\bm{\varphi}_{t,i}}
\newcommand{\phitk}{\bm{\varphi}_{t,k}}
\newcommand{\rvec}{\bm{r}}
\newcommand{\muvec}{\bm{\mu}}
\newcommand{\pivec}{\pi}
\newcommand{\bvec}{\bm{\beta}}

\newcommand{\Em}{\bm{E}}
\newcommand{\Pm}{\bm{P}}
\newcommand{\psivec}[1][x']{\bm{\psi}({#1})}
\newcommand{\pfunc}[2][x']{p({#1}\SetSymbol{#2})}
\newcommand{\Vm}{\bm{\Lambda}}
\newcommand{\fvec}{\bm{\varphi}}
\newcommand{\Dphi}{D_{\vec{\varphi}}}
\renewcommand{\succeq}{\geq}

\DeclarePairedDelimiterXPP\Ipt[2]{{#1}\transpose}(){}{#2}
\DeclarePairedDelimiterXPP\Iptr[2]{}(){\transpose{#2}}{#1}

\newcommand{\Lag}{\mathfrak{L}}
\newcommand{\Unif}[1][n]{\mathcal{U}({#1})}

\newcommand{\softmax}{\sigma}

\newcommand{\Reals}{\mathbb{R}}
\newcommand{\gt}{\bm{g}_{\theta,t}}
\newcommand{\gb}{\bm{g}_{\beta,t}}

\newcommand{\tildegt}{\tilde{\bm{g}}_{\tvec,t,i}}
\newcommand{\tildegtk}{\tilde{\bm{g}}_{\tvec,t,k}}
\newcommand{\tildegb}{\tilde{\bm{g}}_{\bvec,t}}

\newcommand{\maximize}{\mathrm{maximize}}
\newcommand{\minimize}{\mathrm{minimize}}
\newcommand{\subjectto}{\mathrm{subject\ to}}
\newcommand{\lagargs}[1][\bvec]{{#1},\muvec;\vvec,\tvec}
\newcommand{\fargs}[1][\bvec]{{#1},\pivec;\tvec}

\newcommand{\Vars}[2]{\mathbb{V}_{#2}\left[#1\right]}
\usepackage{xspace}
\newcommand{\regret}{\mathfrak{R}}

\newcommand{\rtheta}{{\bm{\omega}}}

\renewcommand{\mid}{|}

\newcommand{\V}[1]{\bm{\Lambda}^{#1}}
\renewcommand{\vec}[1]{{\boldsymbol{{#1}}}}

\newcommand{\F}{\mathcal{F}}

\newcommand{\Z}{\mathcal{Z}}
\newcommand{\GG}{\mathcal{G}}
\newcommand{\bb}[2]{\mathbb{B}(#2)}

\newcommand{\Rn}[0]{\mathbb{R}} % real numbers

\newcommand{\Sw}{\mathcal{S}}

\newcommand{\DD}{\mathcal{D}}

\newcommand{\EE}[1]{\mathbb{E}\left[#1\right]}

\newcommand{\EEs}[2]{\mathbb{E}_{#2}\left[#1\right]}

\newcommand{\EEt}[1]{\mathbb{E}_t\left[#1\right]}

\newcommand{\EEcpi}[2]{\mathbb{E}_\pi\left[#1\middle|#2\right]}
\newcommand{\EEccpi}[2]{\mathbb{E}_\pi\bigl[#1\bigm|#2\bigr]}

\newcommand{\EEti}[1]{\mathbb{E}_{t,i}\left[#1\right]}
\newcommand{\EEtk}[1]{\mathbb{E}_{t,k}\left[#1\right]}
\newcommand{\EET}[1]{\mathbb{E}_T\left[#1\right]}

\newcommand{\EEc}[2]{\mathbb{E}\left[#1\left|#2\right.\right]}
\newcommand{\EEcc}[2]{\mathbb{E}\left[\left.#1\right|#2\right]}

\newcommand{\iprod}[2]{\left\langle#1,#2\right\rangle}

\newcommand{\norm}[1]{\left\|#1\right\|}
\newcommand{\abs}[1]{\left|#1\right|}

\newcommand{\twonorm}[1]{\norm{#1}_2}
\newcommand{\sqtwonorm}[1]{\norm{#1}_2^{2}}
\newcommand{\infnorm}[1]{\norm{#1}_\infty}
\newcommand{\spannorm}[1]{\norm{#1}_{\text{sp}}}

\newcommand{\pa}[1]{\left(#1\right)}
\newcommand{\bpa}[1]{\bigl(#1\bigr)}

\newcommand{\sq}[1]{\pa{#1}^{2}}

\newcommand{\wh}{\widehat}

\newcommand{\lambdamax}{\lambda_{\max}}

\newcommand{\transpose}{^\mathsf{\scriptscriptstyle T}}

\definecolor{PalePurp}{rgb}{0.66,0.57,0.66}

\newcommand{\redd}[1]{\textcolor{red}{#1}}

\newcommand{\DDKL}[2]{\mathcal{D}\pa{#1\middle\|#2}}
\newcommand{\HHKL}[2]{\mathcal{H}\pa{#1\middle\|#2}}

\newcommand{\piout}{\bm{\pi}_{\text{out}}}

\renewcommand{\redd}[1]{{#1}}

\title{Offline Primal-Dual Reinforcement Learning for Linear MDPs}

\author{%
	Germano Gabbianelli \\
	Universitat Pompeu Fabra\\
	Barcelona, Spain \\
	\texttt{germano.gabbianelli@upf.edu} \\
	\And
	Gergely Neu \\
	Universitat Pompeu Fabra\\
	Barcelona, Spain \\
	\texttt{gergely.neu@gmail.com} \\
	\And
	Nneka Okolo \\
	Universitat Pompeu Fabra\\
	Barcelona, Spain \\
	\texttt{nnekamaureen.okolo@upf.edu} \\
	\And
	Matteo Papini \\
	Universitat Pompeu Fabra\\
	Barcelona, Spain \\
	\texttt{matteo.papini@upf.edu} \\
}

\begin{document}

	\maketitle

	%%%%BODY
	\begin{abstract}
	Offline Reinforcement Learning (RL) aims to learn a near-optimal
    policy from a fixed dataset of transitions collected by another policy.
    This problem has attracted a lot of attention recently, but most existing
    methods with strong theoretical guarantees are restricted to finite-horizon
    or tabular settings. In constrast, few algorithms for
    infinite-horizon settings with function approximation and minimal assumptions
    on the dataset are both sample and computationally efficient.
    Another gap in the current literature is the lack of theoretical analysis for
    the average-reward setting, which is more challenging than the discounted setting.
    In this paper, we address both of these issues by proposing a primal-dual
    optimization method based on the linear programming formulation of RL. Our key contribution is a new reparametrization that allows us to derive low-variance gradient estimators that can be used in a stochastic optimization scheme using only samples from the behavior policy. Our method finds an $\varepsilon$-optimal policy with $O(\varepsilon^{-4})$ samples, improving on the previous $O(\varepsilon^{-5})$, while being computationally efficient for infinite-horizon discounted and average-reward MDPs with realizable linear function approximation and partial coverage. Moreover, to the best of our knowledge, this is the first theoretical result for average-reward offline RL.
\end{abstract}

\section{Introduction}

We study the setting of Offline Reinforcement Learning (RL), where the goal is
to learn an $\varepsilon$-optimal policy without being able to interact with
the environment, but only using a fixed dataset of transitions collected
by a \emph{behavior policy}.
Learning from offline data proves to be useful especially when interacting
with the environment can be costly or dangerous \citep{Levine2020}.

In this setting, the quality of the best policy learnable by any
algorithm is constrained by the quality of the data, implying that finding an
optimal policy without further assumptions on the data is not feasible.
Therefore, many methods \citep{munos2008finite, Uehara20} make a \emph{uniform coverage}
assumption, requiring that the behavior policy explores sufficiently well the
whole state-action space. However, recent work~\citep{liu2020provably,
rashidinejad2022bridging} demonstrated that \emph{partial coverage}
of the state-action space is sufficient. In particular, this means
that the behavior policy needs only to sufficiently explore the state-actions
visited by the optimal policy.

Moreover, like its online counterpart, modern offline RL faces the problem of learning
efficiently in environments with
very large state spaces, where function approximation is necessary to compactly
represent policies and value functions. Although function approximation,
especially with neural networks, is widely used in practice, its
theoretical understanding in the context of decision-making is still rather
limited, even when considering \emph{linear} function approximation.

In fact, most existing sample complexity results for offline RL algorithms are
limited either to the tabular and finite horizon setting, by the uniform
coverage assumption, or by lack of computational 
efficiency --- see the top section of Table~\ref{tab:sota} for a summary.
Notable exceptions are the recent works of \citet{Xie21} and \citet{Cheng22}
who provide computationally efficient methods for infinite-horizon discounted
MDPs under realizable linear function approximation and partial coverage.
Despite being some of the first implementable algorithms, their methods work
only with discounted rewards, have
superlinear computational complexity and find an $\varepsilon$-optimal policy with
$O(\varepsilon^{-5})$ samples -- see the bottom section of Table~\ref{tab:sota} for more details. Therefore, this work is motivated by the following research question:

\emph{Can we design a linear-time algorithm with polynomial sample complexity
for the discounted and average-reward infinite-horizon settings, in large
state spaces under a partial-coverage assumption?}
\renewcommand{\tabularxcolumn}[1]{m{#1}}
\renewcommand{\arraystretch}{1.5}
\newcommand{\nope}{\cellcolor{lightgray}\xmark}
\newcolumntype{b}{>{\hsize=1.5\hsize}X}
\newcolumntype{s}{>{\hyphenpenalty=10000\exhyphenpenalty=10000\centering\arraybackslash}X}
\begin{table}\scriptsize
	\label{tab:sota}
    \begin{tabularx}{\textwidth}{>{\hsize=1.6\hsize\linewidth=\hsize}s|*{6}{>{\hsize=.89\hsize\linewidth=\hsize}s}}
        \multirow[c]{2}{\hsize}{\quad\quad\bfseries Algorithm}&
        \multirow[c]{2}{\hsize}{\bfseries Partial Coverage}&
        \multirow[c]{2}{\hsize}{\bfseries Polynomial Sample Complexity}&
        \multirow[c]{2}{\hsize}{\bfseries Polynomial Computational Complexity}&
        \multirow[c]{2}{\hsize}{\bfseries Function Approximation} & \multicolumn{2}{c@{}}{\bfseries Infinite Horizon} \\
		 & & & & & \bfseries Discounted & \bfseries Average-Reward \\ \hline
		FQI~\citep{munos2008finite}         & \nope  &\cmark& \cmark & \cmark & \cmark &\nope \\ 
		\citet{rashidinejad2022bridging}     & \cmark &\cmark& \cmark & \nope  & \cmark &\nope \\ 
		\mbox{\citet{jin2021pessimism}}
        \mbox{\citet{zanette2021provable}}   & \cmark &\cmark& \cmark & \cmark & \nope  &\nope \\ 
		\citet{uehara2022pessimistic}        & \cmark &\cmark& \nope  & \cmark & \cmark &\nope \\ 
        \hline
        \citet{Cheng22}                       & \cmark & $O(\varepsilon^{-5})$ & superlinear  & \cmark & \cmark &\nope \\ 
        \citet{Xie21}                       & \cmark & $O(\varepsilon^{-5})$ & $O(n^{7/5})$  & \cmark & \cmark &\nope \\ 
        \bfseries Ours                      & \cmark & $O(\varepsilon^{-4})$& $O(n)$ & \cmark & \cmark &\cmark 
	\end{tabularx}
	\caption{Comparison of existing offline RL algorithms. The table is divided
    horizontally in two sections. The upper section qualitatively compares algorithms
    for easier settings, that is, methods for
    the tabular or finite-horizon settings or methods which require uniform coverage. The lower section focuses on the setting considered in this paper, that is
    computationally efficient methods for the infinite horizon setting with
    function approximation and partial coverage. }
\end{table}

We answer this question positively by designing a method based on the
linear-programming (LP) formulation of sequential
decision making~\citep{manne1960linear}. Albeit less known than the
dynamic-programming formulation~\citep{bellman1956dynamic} that is ubiquitous
in RL, it allows us to tackle this problem with the powerful tools of convex
optimization. We turn in particular to a relaxed version of the LP formulation~\citep{mehta2009q,BasSerrano2021}
that considers action-value functions that are linear in known state-action features.
This allows to reduce the dimensionality of the problem from the cardinality of the
state space to the number of features. This relaxation still allows to recover
optimal policies in \emph{linear MDPs}~\citep{Yang2019,Jin2020}, a structural
assumption that is widely employed in the theoretical study of RL with linear
function approximation.

Our algorithm for learning near-optimal policies from offline data is based on
primal-dual optimization of the Lagrangian of the relaxed LP. The use of
saddle-point optimization in MDPs was first proposed by~\citet{wang2016online}
for \emph{planning} in small state spaces, and was extended to linear function
approximation by~\citet{chen2018scalable,bas2020faster}, and~\citet{neu2023efficient}.
We largely take inspiration from this latter work, which was the first to apply
saddle-point optimization to the
\emph{relaxed} LP. However, primal-dual planning algorithms assume oracle access
to a transition model, whose samples are used to estimate gradients.
In our offline setting, we only assume access to i.i.d. samples generated by a
possibly unknown behavior policy. To adapt the primal-dual
optimization strategy to this setting we employ a change of variable, inspired
by~\citet{Nachum2020}, which allows easy computation of unbiased gradient estimates.
	\paragraph{Notation.}
We denote vectors with bold letters, such as $\bm{x}\doteq [x_1,\dots,x_d]^\top\in\Reals^d$,
and use $\bm{e}_i$ to denote the $i$-th standard basis vector.
We interchangeably denote functions $f:\X\to\R$ over a finite set $\X$, as
vectors $\bm{f}\in\R^{|\X|}$ with components $f(x)$,
and use $\succeq$ to denote element-wise comparison.
We denote the set of probability distributions over a measurable set $\Sw$ as
$\Delta_{\Sw}$, and the probability simplex in $\Reals^d$ as $\Delta_d$.
We use $\softmax:\Reals^d\to\Delta_d$ to denote the softmax function
defined as $\sigma_i(\bm{x})\doteq e^{x_i}/\sum_{j=1}^d e^{x_j}$.
We use upper-case letters for random variables, such as $S$, and denote the
uniform distribution over a finite set of $n$ elements as $\Unif$.
In the context of iterative algorithms, we use $\F_{t-1}$ to 
denote the sigma-algebra generated by all events up to the end of iteration
$t-1$, and use the shorthand notation $\EEt{\cdot} = \EEcc{\cdot}{\F_{t-1}}$ to
denote expectation conditional on the history.
For nested-loop algorithms, we write $\F_{t,i-1}$ for
the sigma-algebra generated by all events up to the end of iteration $i-1$ of
round $t$, and $\EEti{\cdot} = \EEcc{\cdot}{\F_{t,i-1}}$ for the corresponding conditional expectation.

\section{Preliminaries}\label{sec:prelim}
We study discounted Markov decision processes~\citep[MDP, ][]{Puterman1994}
denoted as $(\X, \A, p, r, \df)$, with discount factor $\df\in[0,1]$ and finite, but potentially
very large, state space $\X$ and action space $\A$. For every state-action pair $(x,a)$, we denote as $\pfunc[\cdot]{x,a}\in\Delta_\X$ the next-state
distribution, and as $r(x,a)\in[0,1]$ the reward, which is assumed to be
deterministic and bounded for simplicity. The transition function $p$ is also
denoted as the matrix $\Pm \in\R^{|\X\times\A|\times|\X|}$ and the reward as
the vector $\rvec\in\R^{|\X\times\A|}$. The objective is to find an \emph{optimal
policy} $\pi^*:\X\to\Delta_{\A}$. That is, a stationary policy that maximizes the normalized expected return
$\rho(\pi^*)\doteq (1-\df)\Es{\pi^*}{\suminfty r(X_t,A_t)}$, where the initial
state $X_0$ is sampled from the initial state distribution $\nu_0$, the other
states according to $X_{t+1}\sim p(\cdot\mid X_t,A_t)$ and where the notation
$\Es{\pi}{\cdot}$ is used to denote that the actions are sampled from policy
$\pi$ as $A_t\sim\pi(\cdot\mid X_t)$. Moreover, we define the following
quantities for each policy $\pi$: its state-action value function
$q^\pi(x,a)\doteq\Es{\pi}{\suminfty \df^t r(X_t,A_t) \given X_0=x, A_0=a}$,
its value function $v^\pi(x)\doteq \Es{\pi}{q^\pi(x,A_0)}$,
its state occupancy measure $\nu^\pi(x)\doteq (1-\df)\Es{\pi}{\suminfty \One{X_t=x}}$, and
its state-action occupancy measure $\mu^\pi(x,a)\doteq \pi(a\mid x)\nu^\pi(x)$.
These quantities are known to satisify the following useful relations, more commonly known
respectively as Bellman's equation and flow constraint for policy $\pi$ \citep{Bellman1966}:
\begin{equation}
    \bm{q}^\pi = \rvec + \df\Pm\vvec^\pi \quad\quad\quad
    \nuvec^\pi = (1-\df)\nuvec_0 + \df\Pm\transpose\muvec^\pi
\end{equation}
Given this notation, we can also rewrite the normalized expected return in
vector form as $\rho(\pi)=(1-\df)\ip{\nuvec_0}{\vvec^\pi}$ or equivalently as
$\rho(\pi)=\ip{\rvec}{\muvec^\pi}$.

Our work is based on the linear programming formulation due to \citet{Manne1960} (see also \citealp{Puterman1994}) 
which transforms the reinforcement learning problem into
the search for an optimal state-action occupancy measure, obtained by solving
the following Linear Program (LP):
\begin{equation}\label{eq:original-lp}
\begin{alignedat}{2}
    & \maximize  &\quad& \ip{\rvec}{\muvec} \\
    & \subjectto && \Em\transpose\muvec =(1-\df)\nuvec_0 + \df\Pm\transpose\muvec \\
                &&& \muvec \succeq 0
\end{alignedat}
\end{equation}
where $\Em\in\R^{|\X\times\A|\times|\X|}$ denotes the matrix with components
$\Em_{(x,a),x'}\doteq\One{x=x'}$. The constraints of this LP are known to
characterize the set of valid state-action occupancy measures. Therefore, an
optimal solution $\muvec^*$ of the LP corresponds to the state-action
occupancy measure associated to a policy $\pivec^*$ maximizing the expected return,
and which is therefore optimal in the MDP. This policy can be extracted
as $\pi^*(a\mid x)\doteq \mu^*(x,a)/\sum_{\bar{a}\in\A} \mu^*(x,\bar a)$. However, this linear program cannot be directly
solved in an efficient way in large MDPs due to the number of
constraints and dimensions of the variables scaling with the size of the state
space $\X$. Therefore, taking inspiration from the previous works of \citet{BasSerrano2021, neu2023efficient} we assume the knowledge of a \emph{feature map} $\varphi$, which
we then use to reduce the dimension of the problem. More specifically we consider the
setting of Linear MDPs \citep{Jin2020,Yang2019}.
\begin{definition}[Linear MDP]\label{def:linMDP} An MDP is called linear if both the transition
    and reward functions can be expressed as a linear function of a given
    feature map $\varphi:\X\times\A\to\R^d$. That is, there exist $\psi:\X\to\R^d$
    and $\rtheta\in\R^d$ such that, for every $x,x'\in\X$ and $a\in\A$:
    \begin{equation*}
        r(x,a) = \ip{\feat}{\rtheta}, \quad\quad\quad \pfunc{x,a} = \ip{\feat}{\psivec}.
    \end{equation*}
    \redd{We assume that for all $x,a$, the norms of all relevant vectors are bounded by known constants as 
$\twonorm{\feat}\le D_{\fvec}$, $\twonorm{\sum_{x'} \psivec} \le D_{\bm{\psi}}$, and $\twonorm{\rtheta} \le 
D_{\bm{\rtheta}}$}. Moreover, we
    represent the feature map with the matrix
    $\Phim\in\R^{|\X\times\A|\times d}$ with rows given by $\feat\transpose$,
    and similarly we define $\Psim\in\R^{d\times|\X|}$ as the matrix with columns given by $\psivec[x]$.
\end{definition}
With this notation we can rewrite the transition matrix as $\Pm = \Phim\Psim$.
Furthermore, it is convenient to assume that the dimension $d$ of the feature map
cannot be trivially reduced, and therefore that the matrix $\Phim$ is full-rank.
An easily verifiable consequence of the Linear MDP assumption is that state-action value functions can be represented as a linear combinations of $\varphi$. That is, there exist $\tvec^\pi\in\R^d$ such that:
\begin{equation}\label{eq:linear-q}
    \bm{q}^\pi = \rvec + \df\Pm\vvec^\pi = \Phim(\rtheta + \Psim\vvec^\pi) = \Phim\tvec^\pi.
\end{equation}
\redd{It can be shown that for all policies $\pi$, the norm of $\tvec^\pi$ is at most $D_{\tvec} = D_{\rtheta} + 
\frac{D_{\bm{\psi}}}{1-\gamma}$ (cf.~Lemma~B.1 in \citealp{Jin2020}).}
We then translate the linear program (\ref{eq:original-lp}) to our setting, with the addition of
the new variable $\lvec\in\R^d$, resulting in the following new LP and its corresponding dual:
\vspace{-2em}
\begin{multicols}{2}
		\begin{equation}\label{eq:LP-phi}
			\begin{alignedat}{2}
				& \maximize  &\quad& \ip{\rtheta}{\lvec} \\
				& \subjectto && \Em\transpose\muvec =(1-\df)\nuvec_0 + \df\Psim\transpose\lvec \\
				&&& \lvec = \Phim\transpose\muvec \\
				&&& \muvec \succeq 0.
			\end{alignedat}
		\end{equation}
	
		\begin{equation}\label{eq:LP-phi-dual}\hspace{-1em}
			\begin{alignedat}{2}
				& \minimize  &\quad& (1-\df)\ip{\nuvec_0}{\vvec} \\
				& \subjectto && \tvec =\rtheta + \df\Psim\vvec \\
				&&& \Em\vvec \succeq \Phim\tvec\\
				&\\
			\end{alignedat}
		\end{equation}
\end{multicols}
\vspace{-1em}

It can be immediately noticed how the introduction of $\lvec$ did not change
neither the set of admissible $\muvec$s nor the objective, and therefore did
not alter the optimal solution. The Lagrangian associated to this set of linear
programs is the function:
\begin{align}
	\Lag(\vvec,\tvec,\lvec,\muvec)
	\nonumber&= (1-\df)\ip{\nuvec_0}{\vvec} 
	+ \ip{\lvec}{\rtheta + \gamma \Psim\vvec - \tvec}
	+ \ip{\muvec}{\Phim\tvec - \Em\vvec} \\
	\label{eq:lag}&= \ip{\lvec}{\rtheta} + \ip{\vvec}{(1-\df)\nuvec_0 + \df\Psim\transpose\lvec - \Em\transpose\muvec}
	+ \ip{\tvec}{\Phim\transpose\muvec - \lvec}.
\end{align}
It is known that finding optimal solutions $(\lvec^\star,\muvec^\star)$ and $(\vvec^\star,\tvec^\star)$ for the primal and dual LPs is equivalent
to finding a saddle point $(\vvec^\star,\tvec^\star,\lvec^\star,\muvec^\star)$ of the Lagrangian function \citep{Bertsekas1982}.
In the next section, we will develop primal-dual methods that aim to find approximate solutions to the above 
saddle-point problem, and convert these solutions to policies with near-optimality guarantees. 
	\section{Algorithm and Main Results}\label{sec:main}
This section introduces the concrete setting we study in this paper, and presents our main contributions.

We consider the offline-learning scenario where the agent has access to
a dataset $\D=({W}_t)_{t=1}^n$, collected by a behavior policy $\pi_B$, and composed of $\n$ random observations of 
the form ${W}_t = (X_t^0,X_t,A_t,R_t,X'_t)$.
The random variables $X_t^0,(X_t,A_t)$ and $X'_t$ are sampled, respectively, from the initial-state distribution 
$\nu_0$, the discounted occupancy measure of the behavior policy, denoted as $\mu_B$, and from 
$\pfunc[\cdot]{X_t,A_t}$. Finally, $R_t$ denotes the reward $r(X_t,A_t)$.  We assume that all observations $W_t$ are 
generated independently of each other, and will often use the notation $\fvec_t = \fvec(X_t,A_t)$.

Our strategy consists in finding approximately good solutions for the LPs~\eqref{eq:LP-phi} and \eqref{eq:LP-phi-dual}
using stochastic optimization methods, which require access to unbiased gradient estimates of the Lagrangian 
(Equation~\ref{eq:lag}). The main challenge we need to overcome is constructing suitable estimators based only on 
observations drawn from the behavior policy. We address this challenge by introducing the matrix $\Vm = 
\EEs{\fvec(X,A)\fvec(X,A)\transpose}{X,A\sim\mu_B}$ (supposed to be invertible for the sake of argument for now), and rewriting 
the 
gradient with respect to $\lvec$ as
\begin{align*}
    \nabla_{\lvec} \Lag(\lagargs[\lvec]) &= \rtheta + \gamma \Psim \vvec - \tvec
  = \Vm^{-1}\Vm\pa{\rtheta + \gamma \Psim \vvec - \tvec}
 \\
 &= \Vm^{-1}\EE{\fvec(X_t,A_t)\fvec(X_t,A_t)\transpose\pa{\rtheta + \gamma \Psim \vvec - \tvec}}
 \\
 &= \Vm^{-1}\EE{\fvec(X_t,A_t)\pa{R_t + \gamma \vvec(X_t') - \iprod{\tvec}{\fvec(X_t,A_t)}}}.
\end{align*}
This suggests that the vector within the expectation can be used to build an unbiased estimator of the desired 
gradient. A downside of using this estimator is that it requires knowledge of $\Vm$. However, this can be 
sidestepped by \redd{a reparametrization trick inspired by \citet{Nachum2020}}: introducing the parametrization $\bvec = 
\Vm^{-1} \lvec$, the objective can be rewritten as
\[
\Lag(\lagargs)
= (1-\df)\ip{\nuvec_0}{\vvec}
+ \ip{\bvec}{\Vm\bigl(\rtheta + \gamma \Psim\vvec - \tvec\bigr)}
+ \ip{\muvec}{\Phim\tvec - \Em\vvec}. 
\]
\redd{This can be indeed seen to generalize the tabular reparametrization of \citet{Nachum2020} to the case of linear 
function approximation.} Notably, our linear reparametrization does not change the structure of the saddle-point 
problem, but allows building an unbiased estimator of $\nabla_{\bvec} \Lag(\lagargs)$ without knowledge of $\Vm$ as
\[
    \tilde{\bm{g}}_{\bvec} = \fvec(X_t,A_t)\pa{R_t + \gamma \vvec(X_t') -
\iprod{\tvec}{\fvec(X_t,A_t)}}.
\]
In what follows, we will use the more general parametrization $\bvec = \Lambda^{-c} \lambda$, with 
$c\in\{{1}/{2},1\}$, and construct a primal-dual stochastic optimization method 
that can be implemented efficiently in the offline setting based on the observations above.
Using $c=1$ allows
to run our algorithm without knowledge of $\Vm$, that is, without
knowing the behavior policy that generated the dataset, while using $c={1}/{2}$
results in a tighter bound, at the price of having to assume knowledge
of $\Vm$.

Our algorithm (presented as Algorithm~\ref{alg:algo}) is inspired by the method of \citet{neu2023efficient}, originally 
designed for planning with a generative model. The algorithm has a double-loop structure, where at each 
iteration $t$ we run one step of stochastic gradient ascent for $\bvec$, and also an inner loop which runs $K$ 
iterations of stochastic gradient descent on $\tvec$ making sure that $\ip{\feat}{\tvec_t}$ is a good approximation of 
the true action-value function of $\pivec_t$. 
Iterations of the inner loop are indexed by $k$. The main idea of the algorithm is to compute the unbiased estimators 
$\tildegtk$ and $\tildegb$ of the gradients $\nabla_{\tvec} \Lag(\bvec_t,\muvec_t;\cdot,\tvec_{t,k})$ and 
$\nabla_{\bvec} \Lag(\bvec_t,\cdot;\vvec_t,\tvec_{t})$, and use them to update the respective 
variables iteratively.
We then define a softmax policy $\pivec_t$ at each iteration $t$ using the $\tvec$ parameters as $\pi_t(a\mid x) = \softmax\left(\alpha\sum_{i=1}^{t-1} \ip{\feat}{\tvec_i}\right)$. The other higher-dimensional variables ($\muvec_t,\vvec_t$) are defined \redd{symbolically} in terms of $\bvec_t$, 
$\tvec_t$ and $\pivec_t$, and used only as auxiliary variables for computing the estimates $\tildegtk$ and 
$\tildegb$. 
Specifically, we set these variables as
\begin{align}
	v_t(x)
	&=\suma \pi_t(a\mid x)\ip{\feat}{\tvec_t},\label{eq:def-vt} \\ 
	\mu_{t,k}(x,a)
	&= \pi_t(a\mid x)\bigl((1-\df)\One{X_{t,k}^0=x}
	+ \df\ip{\fvec_{t,k}}{\Vm^{c-1}\bvec_t}\One{X'_{t,k}=x}\bigr). \label{eq:def-mut} 
\end{align}
Finally, the gradient estimates can be defined as
\begin{align}
    \tildegb &= \Vm^{c-1} \fvec_t\left(R_t + \df v_t(X'_t) - \ip{\fvec_t}{\tvec_t} \right),\label{eq:def-tildegb} \\
    \tildegtk &= \Phim\transpose\muvec_{t,k} - \Vm^{c-1}\fvec_{t,k}\ip{\fvec_{t,k}}{\bvec_t}.\label{eq:def-tildegt}
\end{align}
These gradient estimates are then used in a projected gradient ascent/descent scheme, with the $\ell_2$ projection 
operator denoted by $\Pi$. The feasible sets of the two parameter vectors are chosen as $\ell_2$ balls of 
radii $D_\theta$ and $D_\beta$, denoted respectively as ${\bb{d}{D_{\theta}}}$ and ${\bb{d}{D_{\beta}}}$.
\redd{Notably, the algorithm does not need to compute $v_t(x)$, $\mu_{t,k}(x,a)$, or $\pi_t(a|x)$ for all states $x$, 
but only for the states that are accessed during the execution of the method. In particular, $\pivec_t$ does not need 
to be computed explicitly, and it can be efficiently represented by the single $d$-dimensional parameter vector 
$\sum_{i=1}^t \tvec_i$.}

Due to the double-loop structure, each iteration $t$ uses $K$ samples from the dataset $\D$, adding up to a total of 
$n=KT$ samples over the course of $T$ iterations. Each gradient update calculated by the method uses a constant 
number of elementary vector operations, resulting in a total computational complexity of $O(|\A|dn)$ elementary
operations. At the end, our algorithm outputs a policy selected uniformly at random from the $T$ iterations.
 
\begin{algorithm}[t]
	\caption{Offline Primal-Dual RL}\label{alg:algo}
    \begin{algorithmic}
        \STATE {\bfseries Input:} Learning rates $\alpha,\zeta,\eta$,
                initial points $\tvec_0\in\bb{d}{D_{\theta}},\bvec_1\in\bb{d}{D_{\beta}}, \pi_1$, and data $\D=(W_t)_{t=1}^n$
		\FOR{$t=1$ {\bfseries to} $T$}
        	\STATE Initialize $\tvec_{t,1} = \tvec_{t-1}$
	        \FOR{$k=1$ {\bfseries to} $K-1$}
		        \STATE Obtain sample $W_{t,k}=(X^0_{t,k},X_{t,k},A_{t,k},X'_{t,k})$
		        \STATE $\muvec_{t,k} = \pivec_t\circ\bigl[ (1-\df)\bm{e}_{X^0_{t,k}}+\,\df\ip{\fvec(X_{t,k},A_{t,k})}{\Vm^{c-1}\bvec_t} \bm{e}_{X'_{t,k}} \bigr]$
		        \STATE $\tildegt = \Phim\transpose\muvec_{t,k}-\, \Vm^{c-1}\fvec(X_{t,k},A_{t,k})\ip{\fvec(X_{t,k},A_{t,k})}{\bvec_t}$
		        \STATE $\tvec_{t,k+1} = \Pi_{\bb{d}{D_{\theta}}} (\tvec_{t,k} - \eta \tildegt)\quad$ \emph{// Stochastic gradient descent}
	        \ENDFOR
	        \STATE $\tvec_{t} = \frac{1}{K}\sum_{k=1}^K \tvec_{t,k}$
	        \STATE
	        \STATE Obtain sample $W_t = (X^0_t,X_t,A_t,X'_t)$
	        \STATE $\vvec_t = \Em\transpose\bigl(\pivec_t\circ\Phim\tvec_t\bigr)$
	        \STATE $\tildegb = \fvec(X_t,A)\bigl(R_t + \gamma \vvec_t(X'_t) - \ip{\fvec(X_t,A_t)}{\tvec_t}\bigr)$
	        \STATE $\bvec_{t+1} = \Pi_{\bb{d}{D_{\beta}}} (\bvec_t + \zeta \tildegb)\quad$ \emph{// Stochastic gradient ascent}
	        \STATE
	
	        \STATE $\pivec_{t+1} = \softmax(\alpha \sum_{i=1}^{t} \Phim\tvec_i)\quad$ \emph{// Policy update}
		\ENDFOR
        \STATE {\bfseries return} $\pivec_J$ with $J\sim \Unif[T]$.
	\end{algorithmic}
\end{algorithm}

\subsection{Main result}
We are now almost ready to state our main result. Before doing so, we first need to discuss the 
quantities appearing in the guarantee, and provide an intuitive explanation for them.\looseness=-1

Similarly to previous work, we capture the partial coverage assumption by
expressing the rate of convergence to the optimal policy in terms of a
\emph{coverage ratio} that measures the mismatch between the behavior and the
optimal policy. Several definitions of coverage ratio are surveyed by
\citet{uehara2022pessimistic}. In this work, we employ a notion of \emph{feature}
coverage ratio for linear MDPs that defines coverage in feature space rather
than in state-action space, similarly to~\citet{jin2021pessimism},
but with a smaller ratio.

\begin{definition}\label{def:gcr}
    Let $c\in\{\nicefrac{1}{2},1\}$. We define the generalized coverage ratio as
	\begin{equation*}
        C_{\varphi,c}(\pi^*;\pi_B) = \Es{(X^*,A^*)\sim\mu^{\pi^*}}{\feat[X^*,A^*]}^\top\Vm^{-2c}\E{\feat[X^*,A^*]}.
	\end{equation*}
\end{definition}
We defer a detailed discussion of this ratio to Section~\ref{sec:discussion}, where we compare it with similar notions in the literature. We are now ready to state our main result.

\begin{theorem}\label{thm:main_disc}
    Given a linear MDP (Definition~\ref{def:linMDP}) such that $\tvec^{\pi}\in \bb{d}{D_{\tvec}}$ for any policy $\pi$.
	Assume that the coverage ratio is bounded $C_{\varphi,c}(\pi^*;\pi_B) \le D_{\bvec}$. Then, for any comparator policy $\pi^*$,
    the policy output by an appropriately tuned instance of Algorithm~\ref{alg:algo} satisfies
	$\EE{\ip{\muvec^{\pi^*}-\muvec^{\piout}}{\rvec}} \le \varepsilon$
	with a number of samples $n_\epsilon$ that is
    $O\left({\varepsilon^{-4}}{D_{\tvec}^4D_{\fvec}^{8c}D_{\bvec}^4 d^{2-2c} \log|\A|}\right)$.
\end{theorem}
The concrete parameter choices are detailed in the full version of the theorem
in Appendix~\ref{Gen_disc}.
The main theorem can be simplified by making some standard assumptions, formalized
by the following corollary.
\begin{corollary}\label{chalf}
    Assume that the bound of the feature vectors $D_{\fvec}$ is of order $O(1)$,
    that $D_{\rtheta}=D_{\bm{\psi}}=\sqrt{d}$ and that
    $D_{\bvec} = c\cdot C_{\varphi,c}(\pi^*;\pi_B)$ for some positive universal constant $c$. Then, under the same
    assumptions of Theorem \ref{thm:main_disc}, $n_{\varepsilon}$ is of order
    $O\Bigl(\frac{d^{4} C_{\varphi,c}(\pi^*;\pi_B)^2 \log|\A|}{d^{2c}(1-\gamma)^4 \varepsilon^4}\Bigr)$.
\end{corollary}

	\section{Analysis}\label{sec:analysis}
This section explains the rationale behind some of the technical choices of
our algorithm, and sketches the proof of our main result.

First, we explicitly rewrite the expression of the Lagrangian (\ref{eq:lag}), after
performing the change of variable $\lvec=\Vm^c\bvec$:
\begin{align}
	\Lag(\lagargs)
	&= (1-\df)\ip{\nuvec_0}{\vvec}
	+ \ip{\bvec}{\Vm^c\bigl(\rtheta + \gamma \Psim\vvec - \tvec\bigr)}
	+ \ip{\muvec}{\Phim\tvec - \Em\vvec} \label{eq:lag-beta-p}\\
	&= \ip{\bvec}{\Vm^c\rtheta} + \ip{\vvec}{(1-\df)\nuvec_0 + \df\Psim\transpose\Vm^c\bvec - \Em\transpose\muvec}
	+ \ip{\tvec}{\Phim\transpose\muvec - \Vm^c\bvec}.\label{eq:lag-beta-d}
\end{align}
We aim to find an approximate saddle-point of the above convex-concave objective function. One challenge that we need 
to face is that the variables $\vvec$ and $\muvec$ have dimension proportional to the size of the state space $|\X|$, 
so making explicit updates to these parameters would be prohibitively expensive in MDPs with large state spaces. To address this challenge, we choose to parametrize $\muvec$ in terms of a policy $\pivec$ and $\bvec$ through the 
symbolic assignment $\muvec=\muvec_{\bvec,\pi}$, where
\begin{equation}\label{eq:def-mu}
	\mu_{\bvec,\pivec}(x,a) \doteq \pi(a|x)\Bigl[(1-\df)\nu_0(x) + \df\ip{\psivec[x]}{\Vm^c\bvec}\Bigr].
\end{equation}
This choice can be seen to satisfy the first constraint of the primal LP~\eqref{eq:LP-phi}, and thus the gradient of 
the Lagrangian~\eqref{eq:lag-beta-d} evaluated at $\muvec_{\bvec,\pi}$ with respect to $\vvec$ can be verified to be 
$0$. This parametrization makes it
possible to express the Lagrangian as a function of only $\tvec,\bvec$
and $\pi$ as
\begin{alignat}{2}\label{eq:f-theta}
&f(\tvec, \bvec,\pivec) \doteq \Lag(\bvec, \muvec_{\bvec,\pi}; \vvec, \tvec) = \ip{\bvec}{\Vm^c\rtheta} 
+\ip{\tvec}{\Phim\transpose\muvec_{\bvec,\pivec} - \Vm^c\bvec}.
\end{alignat}
For convenience, we also define the quantities
$\nuvec_{\bvec}=\Em\transpose\muvec_{\bvec,\pivec}$ and
$v_{\tvec,\pivec}(s) \doteq \sum_{a}\pi(a|s)\iprod{\tvec}{\feat}$, which enables
us to rewrite $f$ as
\begin{align}
    f(\tvec, \bvec,\pivec)
    &= \ip{\Vm^c\bvec}{\rtheta - \tvec} + \ip{\vvec_{\tvec,\pivec}}{\nuvec_{\bvec}}
    =(1-\df)\ip{\nuvec_0}{\vvec_{\tvec,\pivec}}
    +\ip{\Vm^c\bvec}{\rtheta + \gamma \Psim \vvec_{\tvec,\pivec} - \tvec}.\label{eq:f-beta}
\end{align}

The above choices allow us to perform stochastic gradient / ascent over the low-dimensional parameters $\tvec$ and 
$\bvec$ and the policy $\pi$. In order to calculate an unbiased estimator of the gradients, we first observe that the 
choice of $\mu_{t,k}$ in Algorithm~\ref{alg:algo} is an unbiased estimator of $\mu_{\bvec_t,\pivec_t}$:
\begin{align*}
	\EEtk{\mu_{t,k}(x,a)}
	&= \pi_t(a\mid x)\Bigl((1-\df)\P{X_{t,k}^0=x}
	+ \EEtk{\One{X'_{t,k}=x}\ip{\phit}{\Vm^{c-1}\bvec_t}}\Bigr) \\
	&= \pi_t(a\mid x)\Bigl((1-\df)\nu_0(x)
	+ \df\sum_{\bar{x},\bar{a}}\mu_B(\bar{x},\bar{a})p(x\mid \bar{x},\bar{a})
	\feat[\bar{x},\bar{a}]\transpose \Vm^{c-1}\bvec_t\Bigr) \\
	&= \pi_t(a\mid x)\Bigl((1-\df)\nu_0(x) + \df\psivec[x]\transpose\Vm\Vm^{c-1}\bvec_t\Bigr) 
	= \mu_{\bvec_t,\pivec_t}(x,a),
\end{align*}
where we used the fact that $p(x\mid \bar{x},\bar{a})=\ip{\psivec[x]}{\feat[\bar{x},\bar{a}]}$,
and the definition of $\Vm$. This in turn facilitates proving that the
gradient estimate $\tildegtk$, defined in Equation~\ref{eq:def-tildegt}, is
indeed unbiased:
\begin{align*}
    \EEtk{\tildegtk} &= \Phim\transpose\EEtk{\muvec_{t,k}}- \Vm^{c-1}\EEtk{\phitk\phitk\transpose}\bvec_t
    = \Phim\transpose\muvec_{\bvec_t,\pivec_t} - \Vm^c\bvec_t = \nabla_{\tvec}\Lag(\bvec_t,\muvec_t;\vvec_t,\cdot).
\end{align*}
A similar proof is used for $\tildegb$ and is detailed in Appendix~\ref{app:SGD}.

Our analysis is based on arguments by \citet{neu2023efficient}, carefully adapted to the reparametrized 
version of the Lagrangian presented above. The proof studies the following central quantity that 
we refer to as \emph{dynamic duality gap}:
\begin{align}
 &\GG_T(\bvec^*,\pi^*;\tvec^*_{1:T}) \doteq
  \frac{1}{T}\sum_{t=1}^T(f(\bvec^*,\pi^*;\tvec_t) - f(\bvec_t,\pi_t;\tvec^*_t)).
\end{align}
Here, $(\tvec_t,\bvec_t,\pi_t)$ are the iterates of the algorithm, $\tvec^*_{1:T}=(\tvec^*_t)_{t=1}^T$ a sequence of comparators for
$\tvec$, and finally $\bvec^*$ and $\pi^*$ are fixed comparators for $\bvec$ and
$\pi$, respectively. Our first key lemma relates the suboptimality of the output
policy to $\GG_T$ for a specific choice of comparators.

\begin{lemma}\label{lem:suboptimality-duality}
	Let $\tvec^*_t \doteq \tvec^{\pi_t}$, $\pi^*$ be any policy, and
    $\beta^*=\Vm^{-c}\Phim^\top\muvec^{\pi^*}$. Then,
	$\EE{\ip{\muvec^{\pi^*}-\muvec^{\piout}}{\rvec}} = \GG_T\bigl(\bvec^*,\pi^*;\tvec^*_{1:T}\bigr)$.
\end{lemma}
The proof is relegated to Appendix~\ref{app:suboptimality-duality}.
Our second key lemma rewrites the gap $\GG_T$ for \emph{any} choice of comparators as the sum of three regret terms:
\begin{lemma}\label{lem:duality-regret´}
	With the choice of comparators of Lemma~\ref{lem:suboptimality-duality}
	\begin{align*}
		\GG_T(\bvec^*,\pi^*;\tvec^*_{1:T})
		&= \frac{1}{T}\sum_{t=1}^T \ip{\tvec_t-\tvec^*_t}{g_{\tvec,t}}
		+\frac{1}{T}\sum_{t=1}^T\ip{\bvec^*-\bvec_t}{g_{\bvec,t}}\\
		&\quad+ \frac{1}{T}\sum_{t=1}^{T}\sum_{s}\nu^{\pi^*}(s)\sum_{a}(\pi^*(a|s)-\pi_t(a|s))\ip{\tvec_t}{\feat},
	\end{align*}
	where $g_{\tvec,t}=\Phim^\top\muvec_{\bvec_t,\pi_t}-\Vm^c\bvec_t$ and $g_{\bvec,t}=\Vm^c(\rtheta+\df \Psim v_{\tvec_t,\pi_t}-\tvec_t)$.
\end{lemma}
The proof is presented in Appendix~\ref{app:duality-regret´}. To conclude the proof we bound the three terms
appearing in Lemma~\ref{lem:duality-regret´}. The first two of those are bounded using standard gradient descent/ascent 
analysis (Lemmas~\ref{lem:regret-theta-disc} and \ref{lem:regret-beta-disc}), while for the latter we use
mirror descent analysis (Lemma~\ref{lem:regret-pi-disc}).
The details of these steps are reported in Appendix~\ref{app:SGD}.

	\section{Extension to Average-Reward MDPs}\label{sec:AMDP}
In this section, we briefly explain how to extend our approach to offline 
learning in \emph{average reward MDPs}, establishing the first sample complexity result for this setting. 
After introducing the setup, we outline a remarkably simple adaptation of our algorithm along with its 
performance guarantees for this setting. The reader is referred to Appendix~\ref{app:AMDP} for the full details, and to 
Chapter $8$ of \citet{Puterman1994} for a more thorough discussion of average-reward MDPs.

In the average reward setting we aim to optimize the objective $\rho^{\pi}(x) = 
\liminf_{T\rightarrow\infty}\frac{1}{T}\EEccpi{\sum_{t=1}^{T}r(x_{t},a_{t})}{x_{1}=x}$, representing the long-term 
average reward of policy $\pi$ when started from state $x\in\X$. Unlike the discounted setting, the 
average reward criterion prioritizes long-term frequency over proximity of good rewards due to the absence of 
discounting which expresses a preference for earlier rewards. As is standard in the related literature, we will assume 
that $\rho^{\pi}$ is well-defined for any policy and is independent of the start state, and thus will use the same 
notation to represent the scalar average reward of policy $\pi$. Due to the boundedness of the rewards, we clearly have 
$\rho^{\pi}\in[0,1]$. Similarly to the 
discounted setting, it is possible to define quantities analogous to the value and action value functions as the 
solutions to the Bellman equations $\bm{q}^{\pi} = \rvec- \rho^{\pi}\vec{1} + \Pm\vvec^{\pi}$, where $\vvec^{\pi}$ is related to the action-value function as $v^{\pi}(x) = \sum_{a}\pi(a|x)q^{\pi}(x,a)$. We will 
make the following standard assumption about the MDP (see, e.g., Section 17.4 of \citet{MT96}):
\begin{assumption}\label{ass:boundedQ}
For all stationary policies $\pi$, the Bellman equations have a solution $\bm{q}^\pi$ 
satisfying $\sup_{x,a} q^\pi(x,a) - \inf_{x,a} q^\pi(x,a) < D_q$.
\end{assumption}
Furthermore, we will continue to work with the linear MDP assumption of Definition~\ref{def:linMDP}, and will 
additionally make the following minor assumption:
\begin{assumption}\label{ass:Phi}
	The all ones vector $\vec{1}$ is contained in the column span of the feature matrix $\Phim$. Furthermore, let 
$\vec{\varrho}\in\Rn^{d}$ such that for all $(x,a)\in\Z$, $\iprod{\feat}{\vec{\varrho}} =1$. 
\end{assumption}

Using these insights, it is straightforward to derive a linear program akin to~\eqref{eq:original-lp} that characterize 
the optimal occupancy measure and thus an optimal policy in average-reward MDPs. Starting from this formulation and 
proceeding as in Sections~\ref{sec:prelim} and~\ref{sec:analysis}, we equivalently restate this optimization problem as 
finding the saddle-point of the reparametrized Lagrangian defined as follows:
\[
	\Lag(\bvec,\muvec; \rho,\vvec,\tvec)
	= \add{\rho+}\langle\bvec\,, \V{c}  [\rtheta + \Psim\vvec - \tvec \add{-\rho\vec{\varrho}}]\rangle + \langle 
	\muvec\,,\Phim\tvec - \Em\vvec \rangle.
\]
As previously, the saddle point can be shown to be equivalent to an optimal occupancy measure under the assumption that 
the MDP is linear in the sense of Definition~\ref{def:linMDP}. Notice that the above Lagrangian slightly differs from that of the discounted setting in 
Equation~\eqref{eq:lag-beta-p} due to the additional optimization parameter $\rho$, but otherwise our main algorithm 
can be directly generalized to this objective. We present details of the derivations and the resulting algorithm in Appendix~\ref{app:AMDP}. The following theorem states the performance guarantees for this method.
\begin{theorem}\label{thm:AMDP}		
    Given a linear MDP (Definition~\ref{def:linMDP}) satisfying Assumption \ref{ass:Phi}
    and such that $\tvec^{\pi}\in \bb{d}{D_{\tvec}}$ for any policy $\pi$.
	Assume that the coverage ratio is bounded $C_{\varphi,c}(\pi^*;\pi_B) \le D_{\bvec}$. Then, for any comparator policy $\pi^*$,
    the policy output by an appropriately tuned instance of Algorithm~\ref{alg:Main} satisfies
	$\EE{\ip{\muvec^{\pi^*}-\muvec^{\piout}}{\rvec}} \le \varepsilon$
	with a number of samples $n_\epsilon$ that is
    $O\left({\varepsilon^{-4}}{D_{\tvec}^4D_{\fvec}^{12c-2}D_{\bvec}^4 d^{2-2c} \log|\A|}\right)$.
\end{theorem}
As compared to the discounted case, this additional dependence of the sample complexity on  $D_{\fvec}$ is 
due to the extra optimization variable $\rho$. We provide the full proof of this theorem along with further 
discussion in Appendix~\ref{app:AMDP}.
	\section{Discussion and Final Remarks}\label{sec:discussion}
In this section, we compare our results with the most relevant ones from the
literature. Our Table~\ref{tab:sota} can be used as a reference. As a complement
to this section, we refer
the interested reader to the recent work by~\citet{uehara2022pessimistic},
which provides a survey of offline RL methods with their coverage and structural
assumptions. Detailed computations can be found in Appendix~\ref{app:ratios}.

An important property of our method is that it only requires partial coverage.
This sets it apart from classic batch RL methods like FQI~\citep{ernst2005tree,munos2008finite},
which require a stronger uniform-coverage assumption. 
Algorithms working under partial coverage are mostly based on the principle of pessimism. However, our algorithm does not implement any form of explicit pessimism.
We recall that, as shown by~\citet{xiao2021optimality}, pessimism is just one
of many ways to achieve minimax-optimal sample efficiency. 

Let us now compare our notion of coverage ratio to the existing notions previsouly
used in the literature. \citet{jin2021pessimism} (Theorem 4.4) rely on a \emph{feature} coverage ratio which can be written as
\begin{equation}\label{eq:fcr}
C^\diamond(\pi^*;\pi_B) = \EEs{\fvec(X,A)\transpose\Vm^{-1}\fvec(X,A)}{X,A\sim\mu^*}.
\end{equation} 
By Jensen's inequality, our $C_{\varphi,1/2}$ (Definition~\ref{def:gcr}) is never
larger than $C^\diamond$. Indeed, notice how the random features in
Equation~\eqref{eq:fcr} are coupled, introducing an extra variance term w.r.t.
$C_{\varphi,1/2}$. Specifically, we can show that $C_{\varphi,1/2}(\pi^*;\pi_B) = C^\diamond(\pi^*;\pi_B) - \Vars{\Vm^{-1/2}\fvec(X,A)}{X,A\sim\mu^*}$, where $\Vars{Z}{} = \mathbb{E}[\norm{Z-\EE{Z}}^2]$ for a random vector $Z$. 
So, besides fine comparisons with existing notions of coverage ratios, we can regard
$C_{\varphi,1/2}$ as a low-variance version of the standard feature coverage ratio. However, our sample complexity bounds do not fully take advantage of this low-variance
property, since they scale quadratically with the ratio itself, rather than linearly, as is more common in previous work.

To scale with $C_{\varphi,1/2}$, our algorithm requires knowledge of $\Vm$,
hence of the behavior policy. However, so does the algorithm from~\citet{jin2021pessimism}.
\citet{zanette2021provable} remove this requirement at the price of a computationally
heavier algorithm. However, both are limited to the finite-horizon setting. 

\citet{uehara2022pessimistic} and \citet{zhang2022corruption} use a coverage
ratio that is conceptually similar to Equation~\eqref{eq:fcr},
\begin{equation}\label{eq:cr_US}
	C^{\dagger}(\pi^*;\pi_B) = \sup_{y\in\Reals^d}\frac{y\transpose\EEs{\fvec(X,A)\fvec(X,A)\transpose}{X,A\sim\mu^*}y}{y\transpose\EEs{\fvec(X,A)\fvec(X,A)\transpose}{X,A\sim\mu_B}y}.
\end{equation}
Some linear algebra shows that $C^\dagger\le C^\diamond\le dC^{\dagger}$.
Therefore, chaining the previous inequalities we know that $C_{\varphi,1/2}\le C^\diamond\le dC^\dagger$.
It should be noted that the algorithm from~\citet{uehara2022pessimistic} also works
in the representation-learning setting, that is, with unknown features.
However, it is far from being efficiently implementable. The algorithm
from \citet{zhang2022corruption} instead is limited to the finite-horizon
setting.

In the special case of tabular MDPs, it is hard to compare our ratio with existing ones,
because in this setting, error bounds are commonly stated in terms of $\sup_{x,a}\nicefrac{\mu^*(x,a)}{\mu_B(x,a)}$,
often introducing an explicit dependency on the number of states~\citep[e.g.,][]{liu2020provably},
which is something we carefully avoided.
However, looking at how the coverage ratio specializes to the tabular setting can
still provide some insight. With known behavior policy,
$C_{\varphi,1/2}(\pi^*;\pi_B)={\scriptstyle{\sum_{x,a}}}\nicefrac{\mu^{*}(x,a)^2}{\mu_{B}(x,a)}$ is smaller than the more
standard $C^\diamond(\pi^*;\pi_B)={\scriptstyle \sum_{x,a}}\nicefrac{\mu^{*}(x,a)}{\mu_{B}(x,a)}$.
With unknown behavior, $C_{\varphi,1}(\pi^*;\pi_B)={\scriptstyle \sum_{x,a}}(\nicefrac{\mu^{*}(x,a)}{\mu_{B}(x,a)})^2$
is non-comparable with $C^\diamond$ in general, but larger than $C_{\varphi,1/2}$.
Interestingly, $C_{\varphi,1}(\pi^*;\pi_B)$ is also equal to $1+\X^2(\mu^*\|\mu_B)$, where
$\X^2$ denotes the chi-square divergence, a crucial quantity in off-distribution
learning based on importance sampling~\cite{cortes2010learning}.
Moreover, a similar quantity to $C_{\varphi,1}$ was used by~\citet{lykouris2021corruption}
in the context of (online) RL with adversarial corruptions. 

We now turn to the works of \citet{Xie21} and \citet{Cheng22},
which are the only practical methods to consider function approximation in the infinite horizon
setting, with minimal assumption on the dataset, and thus the only directly
comparable to our work. They both
use the coverage ratio $C_{\mathcal{F}}(\pi^*;\pi_B) = \max_{f\in\mathcal{F}}\nicefrac{\norm{f-\mathcal{T}f}_{\mu^*}^2}{\norm{f-\mathcal{T}f}_{\mu_B}^2},$ where $\mathcal{F}$ is a function class and $\mathcal{T}$ is Bellman's operator.
This can be shown to reduce to Equation~\eqref{eq:cr_US} for linear MDPs.
However, the specialized bound of \citet{Xie21} (Theorem 3.2)
scales with the potentially larger ratio from Equation~\eqref{eq:fcr}. Both their algorithms have
superlinear computational complexity and a sample complexity of $O(\varepsilon^{-5})$.
Hence, in the linear MDP setting, our algorithm is a strict improvement both for
its $O(\varepsilon^{-4})$ sample complexity and its $O(n)$ computational complexity.
However, It is very important to notice that no practical algorithm for this
setting so far, including ours, can match the minimax optimal sample complexity rate of
$O(\varepsilon^2)$ \citep{xiao2021optimality,rashidinejad2022bridging}. This leaves space for future work in this area.
In particular, by inspecting our proofs, it should be clear the the extra $O(\varepsilon^{-2})$
factor is due to the nested-loop structure of the algorithm. Therefore, we find
it likely that our result can be improved using optimistic descent methods \citep{borkar1997stochastic} or a two-timescale approach \citep{Kor76,RS13b}.

As a final remark, we remind that when $\Vm$ is unknown, our error bounds scales with $C_{\varphi,1}$, instead of the smaller $C_{\varphi,1/2}$. However, we find it plausible that one can replace the $\Vm$ with an estimate that
is built using some fraction of the
overall sample budget. In particular, in the
tabular case, we could simply use all data to estimate the visitation probabilities
of each-state action pairs and use them to build an estimator of $\Vm$.
Details of a similar approach have been worked out by \citet{Gabbianelli23}.
Nonetheless, we designed our algorithm to be flexible and work in both cases.

To summarize, our method is one of the few not to assume the state space to be
finite, or the dataset to have global coverage, while also being computationally
feasible. Moreover, it offers a significant advantage, both in terms of sample
and computational complexity, over the two existing polynomial-time algorithms
for discounted linear MDPs with partial coverage~\cite{Xie21,Cheng22}; it extends
to the challenging average-reward setting with minor modifications; and has
error bounds that scale with a low-variance version of the typical coverage ratio.
These results were made possible by employing algorithmic principles, based on
the linear programming formulation of sequential decision making, that are new
in offline RL.
Finally, the main direction for future work is to develop a single-loop algorithm
to achieve the optimal rate of $\varepsilon^{-2}$, which should
also improve the dependence on the coverage ratio from $C_{\varphi,c}(\pi^*;\pi_B)^2$ to $C_{\varphi,c}(\pi^*;\pi_B)$.

	%%%%ACKNOWLEDGEMENT
%	\begin{ack}
%		Use unnumbered first level headings for the acknowledgments. All acknowledgments
%		go at the end of the paper before the list of references. Moreover, you are required to declare
%		funding (financial activities supporting the submitted work) and competing interests (related financial activities outside the submitted work).
%		More information about this disclosure can be found at: \url{https://neurips.cc/Conferences/2023/PaperInformation/FundingDisclosure}.
%		
%		
%		Do {\bf not} include this section in the anonymized submission, only in the final paper. You can use the \texttt{ack} environment provided in the style file to autmoatically hide this section in the anonymized submission.
%	\end{ack}

	%%%%REFERENCES	
	%	\section*{References}
	%	\bibliographystyle{natbib}
	\bibliographystyle{classes/icml2023}
	\bibliography{references/references}

	\newpage
	\appendix
	%%%%SUPPLEMENTARY MATERIAL
	\section*{Supplementary Material}\label{supp_material}
	\section{Complete statement of Theorem \ref{thm:main_disc}}\label{Gen_disc}
\begin{theorem}
	Consider a linear MDP (Definition~\ref{def:linMDP}) such that $\tvec^{\pi}\in \bb{d}{D_{\tvec}}$ for all $\pi\in\Pi$. Further, suppose that $C_{\varphi,c}(\pi^*;\pi_B) \le D_{\bvec}$. Then, for any comparator policy $\pi^*\in\Pi$, the policy output by Algorithm~\ref{alg:algo} satisfies:
	\[
	\EE{\ip{\muvec^{\pi^*}-\muvec^{\piout}}{\rvec}}
	\le \frac{2D_{\bvec}^2}{\zeta T}
	+ \frac{\log|\A|}{\alpha T}
	+ \frac{2D_{\tvec}^2}{\eta K} 
	+ \frac{\zeta G_{\bvec,c}^2}{2}
	+\frac{\alpha D_{\tvec}^2D_{\fvec}^2}{2}
	+ \frac{\eta G_{\tvec,c}^2}{2},
	\]
	where:
	\begin{align}
		&G_{\tvec,c}^2 = 3D_{\fvec}^2\left((1-\df)^2+ (1+\df^2) D_{\bvec}^2\norm{\Vm}_2^{2c-1}\right),\label{eq:G1}\\
		&G_{\bvec,c}^2 = 3(1+(1+\df^2)\Dphi^2 D_{\tvec}^2)D_{\fvec}^{2(2c-1)}.\label{eq:G2}
	\end{align}
	In particular, using learning rates $\eta=\frac{2D_{\tvec}}{G_{\tvec,c}\sqrt{K}}$,  
	$\zeta=\frac{2D_{\bvec}}{G_{\bvec,c}\sqrt{T}}$, and $\alpha=\frac{\sqrt{2\log|\A|}}{D_{\fvec}D_{\tvec}\sqrt{T}}$, and 
	setting $K=T\cdot\frac{2D_{\bvec^2}G_{\bvec,c}^2+D_{\tvec}^2D_{\fvec}^2\log|\A|}{2D_{\tvec}^2G_{\tvec,c}^2}$, we 
	achieve $\EE{\ip{\muvec^{\pi^*}-\muvec^{\piout}}{\rvec}}\le \epsilon$ with a number of samples $n_\epsilon$ that is 
	\begin{equation*}
		O\left({\epsilon^{-4}}{D_{\tvec}^4D_{\fvec}^4D_{\bvec}^4 \Tr(\V{2c-1})\norm{\Vm}_2^{2c-1}\log|\A|}\right).
	\end{equation*}
\end{theorem}
By remark~\ref{remark1} below, we have that $n_\epsilon$ is simply of order $		O\left({\varepsilon^{-4}}{D_{\tvec}^4D_{\fvec}^{8c}D_{\bvec}^4 d^{2-2c} \log|\A|}\right)$

\begin{remark}\label{remark1}
		When $c=1/2$, the factor $\mathrm{Tr}(\Vm^{2c-1})$ is just $d$, the feature dimension, and $\norm{\Vm}_2^{2c-1}=1$. When $c=1$ and $\Vm$ is unknown, both $\norm{\Vm}_2$ and $\Tr(\Vm)$ should be replaced by their upper bound $D_{\fvec}^2$. Then, for $c\in\{1/2,1\}$, we have that $\Tr(\V{2c-1})\norm{\Vm}_2^{2c-1}\leq D_{\fvec}^{8c-4}d^{2-2c}$.
\end{remark}

	\newpage
	\section{Missing Proofs for the Discounted Setting}\label{app:discounted}
\subsection{Proof of Lemma~\ref{lem:suboptimality-duality}}\label{app:suboptimality-duality}
Using the choice of comparators described in the lemma, we have
	\begin{align*}
	\nu_{\bvec^*}(s)
	&= (1-\df)\nu_0(s)+\df\ip{\psivec[
		s]}{\Vm^c\Vm^{-c}\Phim^\top\mu^{\pi^*}} \\
	&=(1-\df)\nu_0(s)+\sum_{s',a'}P(s|s',a')\mu^{\pi^*}(s',a') = \nu^{\pi^*}(s),  
	\end{align*}
	hence $\mu_{\bvec^*,\pi^*} = \muvec^{\pi^*}$.
	From Equation~\eqref{eq:f-theta} it is easy to see that
	\begin{align*}
	f(\bvec^*,\pi^*;\tvec_t) &= \ip{\Vm^{-c}\Phim^\top\muvec^*}{\Vm^c\rtheta}
	+ \ip{\tvec_t}{\Phim^\top\muvec^* - \Vm^c\Vm^{-c}\Phim^\top\muvec^*} \\
	&=\ip{\mu^{\pi^*}}{\Phim\rtheta} = \ip{\muvec^*}{\rvec}.
	\end{align*}
	Moreover, we also have
	\begin{align*}
	v_{\tvec_t^*,\pi_t}(s)&= \sum_{a}\pi_t(a|s)\ip{\tvec^{\pi_t}}{\feat} \\
	&=\sum_{a}\pi_t(a|s)q^{\pi_t}(s,a)
	= v^{\pi_t}(s,a).
	\end{align*}
	Then, from Equation~\eqref{eq:f-beta} we obtain
	\begin{align*}
	&f(\tvec^*_t,\bvec_t,\pi_t)\\
	&\quad= (1-\df)\ip{\nuvec_0}{v^{\pi_t}}
	+ \ip{\bvec_t}{\Vm^{c}(\rtheta+\df\Psim \vvec^{\pi_t}-\tvec^{\pi_t})} \\
	&\quad= (1-\df)\ip{\nuvec_0}{v^{\pi_t}}
	+ \ip{\bvec_t}{\Vm^{c-1}\EEs{\fvec(X,A)\fvec(X,A)\transpose(\rtheta+\df\Psim \vvec^{\pi_t}-\tvec^{\pi_t})}{X,A\sim\mu_B}} \\
	&\quad= (1-\df)\ip{\nuvec_0}{v^{\pi_t}}
	+ \ip{\bvec_t}{\Vm^{c-1}\EEs{[r(X,A)+\df\iprod{p(\cdot|X,A)}{\vvec^{\pi_{t}}} - \bm{q}^{\pi_t}(X,A)]\fvec(X,A)}{X,A\sim\mu_B}} \\
	&\quad= (1-\df)\ip{\nuvec_0}{v^{\pi_t}} = \ip{\mu^{\pi_t}}{\rvec},
	\end{align*}
	where the fourth equality uses that the value functions satisfy the Bellman equation $\bm{q}^\pi = \rvec + \df\Pm\vvec^\pi$ for any policy $\pi$. The proof is 
concluded by noticing that, since $\piout$ is sampled uniformly from $\{\pi_t\}_{t=1}^T$, 
$\EE{\ip{\muvec^{\piout}}{\rvec}} = \frac{1}{T}\sum_{t=1}^T\EE{\ip{\muvec^{\pi_t}}{\rvec}}$.
\qed

\subsection{Proof of Lemma~\ref{lem:duality-regret´}}\label{app:duality-regret´}
We start by rewriting the terms appearing in the definition of $\GG_T$:
	\begin{align}
	f(\bvec^*,\pi^*;\tvec_t) - f(\bvec_t,\pi_t;\tvec^*_t)
	&=
	f(\bvec^*,\pi^*;\tvec_t) - f(\bvec^*,\pi_t;\tvec_t) \nonumber\\
	&+f(\bvec^*,\pi_t;\tvec_t) - f(\bvec_t,\pi_t;\tvec_t) \nonumber\\
	&+f(\bvec_t,\pi_t;\tvec_t) - f(\bvec_t,\pi_t;\tvec_t^*).\label{eq:three-regrets}
	\end{align}
	To rewrite this as the sum of the three regret terms, we first note that
	\begin{align*}
		f(\fargs) = \ip{\Vm^c\bvec}{\rtheta-\tvec_t} + \ip{\nu_{\bvec}}{v_{\tvec_t,\pi}},
	\end{align*}
	which allows us to write the first term of Equation~\eqref{eq:three-regrets} as
	\begin{align*}
	f(\bvec^*,\pi^*;\tvec_t) - f(\bvec^*,\pi_t;\tvec_t)
	&= \ip{\Vm^c(\bvec^*-\bvec^*)}{\rtheta-\tvec_t}
	+ \ip{\nu_{\bvec^*}}{v_{\tvec_t,\pi^*}-v_{\tvec_t,\pi_t}} \\
	&=\ip{\nu_{\bvec^*}}{\sum_a(\pi^*(a|\cdot)-\pi_t(a|\cdot))\ip{\tvec_t}{\feat[\cdot,a]}},
	\end{align*}
	and we have already established in the proof of Lemma~\ref{lem:duality-to-suboptimality} that $\nuvec_{\bvec^*}$ is 
equal to $\nuvec^{\pi^*}$ for our choice of comparator. 
	Similarly, we use Equation~\eqref{eq:f-beta} to rewrite the second term of Equation~\eqref{eq:three-regrets} as
	\begin{align*}
	f(\bvec^*,\pi_t;\tvec_t) - f(\bvec_t,\pi_t;\tvec_t)
	&=
	(1-\df)\ip{\nuvec_0}{v_{\tvec_t,\pi_t}-v_{\tvec_t,\pi_t}}
	+\ip{\bvec^*-\bvec_t}{\Vm^c(\rtheta+\df\Psim v_{\tvec_t,\pi_t}-\tvec_t)} \\
	&=\ip{\bvec^*-\bvec_t}{g_{\bvec,t}}.
	\end{align*}
	 Finally, we use Equation~\eqref{eq:f-theta} to rewrite the third term of Equation~\eqref{eq:three-regrets} as
	 \begin{align*}
	 f(\bvec_t,\pi_t;\tvec_t) - f(\bvec_t,\pi_t;\tvec_t^*)
	 &= 
	 \ip{\bvec_t-\bvec_t}{\Vm^c\rtheta}
	 +\ip{\tvec_t-\tvec^*_t}{\Phim^\top\muvec_{\bvec_t,\pi_t}-\Vm^c\bvec_t} \\
	 &=\ip{\tvec_t-\tvec^*_t}{g_{\tvec,t}}.
	 \end{align*}
	 
\subsection{Regret bounds for stochastic gradient descent / ascent}\label{app:SGD}
\begin{lemma}\label{lem:regret-theta-disc}
	For any dynamic comparator $\tvec_{1:T}\in D_{\tvec˘}$, the iterates $\tvec_1,\dots,\tvec_T$ of Algorithm~\ref{alg:algo} satisfy the following regret bound:
	\begin{equation*}
		\EE{\sum_{t=1}^T \ip{\tvec_t-\tvec^*_t}{g_{\tvec,t}}} \le \frac{2TD_{\tvec}^2}{\eta K} + \frac{3\eta T D_{\fvec}^2\left((1-\df)^2+ (1+\df^2) D_{\beta}^2\norm{\Vm}_2^{2c-1}\right)}{2}.
	\end{equation*}
\end{lemma}
\begin{proof}
	First, we use the definition of $\tvec_{t}$ as the average of the inner-loop iterates from Algorithm~\ref{alg:algo}, together with linearity of expectation and bilinearity of the inner product.
	\begin{equation}
	\EE{\sum_{t=1}^T \ip{\tvec_t-\tvec^*_t}{g_{\tvec,t}}} 
=\sum_{t=1}^T\frac{1}{K}\underbrace{\EE{ \sum_{k=1}^K\ip{\tvec_{t,k}-\tvec^*_t}{g_{\tvec,t}}}}_{\regret_t}.\label{eq:regrets}
	\end{equation}
	We then appeal to standard stochastic gradient descent analysis to bound each term $\regret_t$ separately.
	
	We have already proven in Section~\ref{sec:analysis} that the gradient estimator for $\tvec$ is unbiased, that is, $\EEtk{\tildegtk}=\gt$. It is also useful to recall here that $\tildegtk$ does \emph{not} depend on $\tvec_{t,k}$. Next, we show that its second moment is bounded. From Equation~\eqref{eq:def-tildegt}, plugging in the definition of $\mu_{t,k}$ from Equation~\eqref{eq:def-mut} and using the abbreviations $\fvec_{t,k}^0 = \sum_{a}\pi_t(a|x_{t,k}^0)\feat[x_{t,k}^0,a]$, $\fvec_t=\feat[x_{t,k},a_{t,k}]$, and $\fvec_{t,k}'=\sum_{a}\pi_t(a|x_{t,k}^0)\feat[x_{t,k}',a]$, we have:
	\begin{align*}
		&\EEtk{\norm{\tildegt}^2}\\
		&\qquad\quad= 	\EEtk{\norm{(1-\df)\fvec_{t,k}^0+\df\fvec_{t,k}'\ip{\fvec_{tk}}{\Vm^{c-1}\bvec_t}-\fvec_{t,k}\ip{\fvec_{tk}}{\Vm^{c-1}\bvec_t}}^2} \\
		&\qquad\quad\le 3(1-\df)^2\D_{\fvec}^2 + 3\df^2\EEtk{\norm{\fvec_{t,k}'\ip{\fvec_{tk}}{\Vm^{c-1}\bvec_t}}^2}+ 3\EEtk{\norm{\fvec_{t,k}\ip{\fvec_{tk}}{\Vm^{c-1}\bvec_t}}^2}\\
		&\qquad\quad\le 3(1-\df)^2\D_{\fvec}^2 + 3(1+\df^2)D_{\fvec}^2\EEtk{\ip{\fvec_{tk}}{\Vm^{c-1}\bvec_t}^2} \\
		&\qquad\quad= 3(1-\df)^2\D_{\fvec}^2 + 3(1+\df^2)D_{\fvec}^2\bvec_t^\top\Vm^{c-1}\EEtk{\fvec_{tk}\fvec_{tk}^\top}\Vm^{c-1}\bvec_t \\
		&\qquad\quad=3(1-\df)^2\D_{\fvec}^2 + 3(1+\df^2)D_{\fvec}^2\norm{\bvec_t}^2_{\Vm^{2c-1}}.
	\end{align*}
	We can then apply Lemma~\ref{lem:aux-sgd} with the latter expression as $G^2$, $\bb{d}{D_{\tvec}}$ as the domain, and $\eta$ as the learning rate, obtaining:
	\begin{align*}
		\EEt{ \sum_{k=1}^K\ip{\tvec_{t,k}-\tvec^*_t}{g_{\tvec,t}}} &\le \frac{\sqtwonorm{\tvec_{t,1} - \tvec^{*}_t}}{2\eta} + \frac{3\eta K D_{\fvec}^2\left((1-\df)^2 + (1+\df^2)\norm{\bvec_t}^2_{\Vm^{2c-1}}\right)}{2}\\
		&\le \frac{2D_{\theta}^2}{\eta} + \frac{3\eta K D_{\fvec}^2\left((1-\df)^2 + (1+\df^2)\norm{\bvec_t}^2_{\Vm^{2c-1}}\right)}{2}.
	\end{align*}
	Plugging this into Equation~\eqref{eq:regrets} and bounding $\norm{\bvec_t}^2_{\Vm^{2c-1}}\le D_{\bvec}^2\norm{\Vm}_2^{2c-1}$, we obtain the final result.
\end{proof}

\begin{lemma}\label{lem:regret-beta-disc}
	For any comparator $\bvec\in D_{\bvec}$, the iterates $\bvec_1,\dots,\bvec_T$ of Algorithm~\ref{alg:algo} satisfy the following regret bound:
	\begin{equation*}
	\EE{\sum_{t=1}^T \ip{\bvec^*-\bvec_t}{g_{\bvec,t}}} \le \frac{2D_{\bvec}^2}{\zeta} + \frac{3\zeta T(1+(1+\df^2)\Dphi^2 D_{\tvec}^2)\Tr(\V{2c-1})}{2}.
	\end{equation*}
\end{lemma}
\begin{proof}
	We again employ stochastic gradient descent analysis.
	We first prove that the gradient estimator for $\bvec$ is unbiased. Recalling the definition of $\tildegb$ from Equation~\eqref{eq:def-tildegb},
	\begin{align*}
		\EE{\tildegb\mid\F_{t-1},\tvec_t} &= \EE{\Vm^{c-1} \fvec_t\left(R_t + \df v_t(X'_t) - \ip{\fvec_t}{\tvec_t} \right)\mid\F_{t-1},\tvec_t} \\
		&=\Vm^{c-1}\big(\EEt{\fvec_t\fvec_t^\top}\rtheta + \df \EEt{\fvec_tv_t(X_t')} - \EEt{\fvec_t\fvec_t^\top}\tvec_t\big)  \\
		&=\Vm^{c-1}\big(\Vm\rtheta + \df \EEt{\fvec_tv_t(X_t')} - \Vm\tvec_t\big) \\
		&=\Vm^{c-1}\big(\Vm\rtheta + \df \EEt{\fvec_t\Pm(\cdot|X_t,A_t)\vvec_t} - \Vm\tvec_t\big) \\
		&=\Vm^{c-1}\big(\Vm\rtheta + \df \EEt{\fvec_t\fvec_t^\top}\Psim\vvec_t - \Vm\tvec_t\big) \\
		&=\Vm^c(\rtheta+\df \Psim v_{\tvec_t,\pi_t}-\tvec_t) = \gb,
	\end{align*}
	recalling that $\vvec_t=\vvec_{\tvec_t,\pi_t}$.
	Next, we bound its second moment. We use the fact that $r\in[0,1]$ and $\|\vvec_{t}\|_{\infty}\leq \|\Phim\tvec_{t}\|_{\infty}\leq \Dphi 
	D_{\tvec}$ to show that
	\begin{align*}
	\EEc{\|\Tilde{\vec{g}}_{\bvec,t}\|_{2}^{2}}{\F_{t-1},\tvec_{t}}
	&=\EEc{\sqtwonorm{ \V{c-1}\phit[R_{t} + \df v_{t}(X_{t}') - \iprod{\tvec_{t}}{\phit}]}}{\F_{t-1},\tvec_{t}}\\
	&\le3(1+(1+\df^2)\Dphi^2 D_{\tvec}^2)\EEt{\phit\transpose\V{2(c-1)}\phit}\\
	&=3(1+(1+\df^2)\Dphi^2 D_{\tvec}^2)\EEt{\Tr(\V{2(c-1)}\phit\phit\transpose)}\\
	&=3(1+(1+\df^2)\Dphi^2 D_{\tvec}^2)\Tr(\V{2c-1}).\\
	\end{align*}
	Thus, we can apply Lemma~\ref{lem:aux-sgd} with the latter expression as $G^2$, $\bb{d}{D_{\bvec}}$ as the domain, and $\zeta$ as the learning rate.
\end{proof}

\begin{lemma}\label{lem:regret-pi-disc}
	The sequence of policies $\pi_1,\dots,\pi_T$ of Algorithm~\ref{alg:algo} satisfies the following regret bound:
	\begin{equation*}
	 \EE{\sum_{t=1}^{T}\sum_{x\in\X}\nu^{\pi^*}(x)\sum_{a}(\pi^*(a|x)-\pi_t(a|x))\ip{\tvec_t}{\feat[x,a]}} \le \frac{\log|\A|}{\alpha} + \frac{\alpha T D_{\fvec}^2D_{\tvec}^2}{2}.
	\end{equation*}
\end{lemma}
\begin{proof}
	We just apply mirror descent analysis, invoking Lemma~\ref{lem:aux-mirror} with $q_t=\Phi\tvec_t$, noting that $\norm{q_t}_{\infty}\le D_{\fvec}D_{\tvec}$. The proof is concluded by trivially bounding the relative entropy as $\HHKL{\pi^*}{\pi_1}=\EEs{\DDKL{\pi(\cdot|x)}{\pi_1(\cdot|x)}}{x\sim\nu^{\pi^*}}\le \log|\A|$.
\end{proof}
	\newpage
	\section{Analysis for the Average-Reward MDP Setting}\label{app:AMDP}
This section describes the adaptation of our contributions in the main body of the paper to average-reward MDPs 
(AMDPs).
In the offline reinforcement learning setting that we consider, we assume access to a sequence of data points 
$(X_{t},A_{t},R_{t},X_{t}')$ in round $t$ generated by a behaviour policy $\pi_{B}$ whose occupancy measure is denoted 
as $\muvec_{B}$. Specifically, we will now draw i.i.d.~samples from the \emph{undiscounted} occupancy measure as 
$X_t,A_t\sim\muvec_B$, sample $X'_{t} \sim p(\cdot|X_{t},A_{t})$, and compute immediate rewards as $R_{t} = 
r(X_{t},A_{t})$. For simplicity, we use the shorthand notation $\phit = \varphi(X_{t},A_{t})$ to denote the 
feature vector drawn in round $t$, and define the matrix $\Vm=\EE{\varphi(X_t,A_t)\varphi(X_t,A_t)^\top}$.

Before describing our contributions, some definitions are in order. An important central concept in the theory of AMDPs 
is that of the \emph{relative value functions} of policy $\pi$ defined as
\begin{align*}
	&v^{\pi}(x) = \lim_{T\rightarrow\infty} \EEcpi{\sum_{t=0}^{T}r(X_{t},A_{t}) - \rho^{\pi}}{X_{0} = x},\\
	&q^{\pi}(x,a) = \lim_{T\rightarrow\infty}\EEcpi{\sum_{t=0}^{T}r(X_{t},A_{t}) - \rho^{\pi}}{X_{0} = x, A_{0} = 
		a},
\end{align*}
where we recalled the notation $\rho^\pi$ denoting the average reward of policy $\pi$ from the main text. These 
functions are sometimes also called the \emph{bias functions}, and their intuitive role is to measure the total amount of reward gathered by policy $\pi$ before it hits its stationary distribution. For simplicity, we will refer to 
these functions as value functions and action-value functions below.

By their recursive nature, these value functions are also characterized by the corresponding Bellman equations recalled below for completeness
\begin{equation*}
	\bm{q}^{\pi} = \rvec- \rho^{\pi}\vec{1} + \Pm\vvec^{\pi},
\end{equation*}
where $\vvec^{\pi}$ is related to the action-value function as $v^{\pi}(x) = \sum_{a}\pi(a|x)q^{\pi}(x,a)$. We note 
that the Bellman equations only characterize the value functions up to a constant offset. That is, for any policy 
$\pi$, and constant $c\in\Rn$, $\vvec^{\pi} + c\vec{1}$ and $\bm{q}^{\pi} + c\vec{1}$ also satisfy the Bellman 
equations. A key quantity to measure the size of the value functions is the \emph{span seminorm} defined for 
$\bm{q}\in\Rn^{\X\times\A}$ as $\spannorm{\bm{q}} =\sup_{(x,a)\in\X\times\A}q(x,a) - 
\inf_{(x,a)\in\X\times\A}q(x,a)$. Using this notation, the condition of Assumption~\ref{ass:boundedQ} can be simply stated as 
requiring $\spannorm{\bm{q}^{\pi}} \le D_q$ for all $\pi$.

Now, let $\pi^{*}$ denote an optimal policy with maximum average reward and introduce the shorthand notations 
$\rho^{*} = \rho^{\pi^{*}}, \muvec^{*} = \muvec^{\pi^{*}}, \nuvec^{*} = \nuvec^{\pi^{*}}, \vvec^{*} = \vvec^{\pi^{*}}$ 
and $\bm{q}^{*} = \bm{q}^{\pi^{*}}$. Under mild assumptions on the MDP that we will clarify shortly, the following Bellman optimality equations are known to characterize bias vectors corresponding to the optimal policy
\begin{align*}
	\bm{q}^{*}= \rvec - \rho^{*} \vec{1} + \Pm\vvec^{*},
\end{align*}
where $\vvec^{*}$ satisfies $v^{*}(x) = \max_{a} q^{*}(x,a)$. Once again, shifting the solutions by a constant 
preserves the optimality conditions. It is easy to see that such constant offsets do not influence greedy or softmax 
policies extracted from the action value functions. Importantly, by a calculation analogous to 
Equation~\eqref{eq:linear-q}, the action-value functions are exactly realizable under the linear MDP condition (see 
Definition~\ref{def:linMDP}) and Assumption~\ref{ass:Phi}. 

Besides the Bellman optimality equations stated above, optimal policies can be equivalently characterized via the following linear program:
\begin{equation}\label{eq:original-lp-AMDP}
\begin{alignedat}{2}
    & \maximize  &\quad& \ip{\muvec}{\rvec} \\
    & \subjectto && \Em\transpose\muvec =\Pm\transpose\muvec \\
                &&&\iprod{\muvec}{\vec{1}}=  1\\
                &&& \muvec \ge 0.
\end{alignedat}
\end{equation}
This can be seen as the generalization of the LP stated for discounted MDPs in the main text, with the added 
complication that we need to make sure that the occupancy measures are normalized\footnote{This is necessary because 
of the absence of $\nu_0$ in the LP, which would otherwise fix the scale of the solutions.} to $1$. By following the 
same steps as in the main text to relax the constraints and reparametrize the LP, one can show that solutions of the LP 
under the linear MDP assumption can be constructed by finding the saddle point of the following Lagrangian:
\begin{align*}
	\Lag(\lvec, \muvec; \rho,\vvec,\tvec)
	&= \add{\rho+}\langle\lvec\,, \rtheta + \Psim\vvec - \tvec \add{-\rho\vec{\varrho}}\rangle + \langle 
\vec{u}\,,\Phim\tvec - \Em\vvec \rangle\\
	&= \rho[1-\iprod{\lvec}{\bm{\varrho}}] + \iprod{\tvec}{\Phim\transpose\muvec - \lvec} + 
\iprod{\vvec}{\Psim\transpose\lvec - \Em\transpose \muvec}.
\end{align*}
As before, the optimal value functions $\bm{q}^*$ and $\bm{v}^*$ are optimal primal variables for the saddle-point 
problem, as are all of their constant shifts. Thus, the existence of a solution with small span seminorm implies the 
existence of a solution with small supremum norm.

Finally, applying the same reparametrization $\bvec = \Vm^{-c} \lvec$ as in the discounted setting, we arrive to 
the following Lagrangian that forms the basis of our algorithm:
\begin{align*}
	\Lag(\bvec,\muvec; \rho,\vvec,\tvec)= \add{\rho+}\langle\bvec\,, \V{c}  [\rtheta + \Psim\vvec - \tvec 
\add{-\rho\vec{\varrho}}]\rangle + \langle \muvec\,,\Phim\tvec - \Em\vvec \rangle.
\end{align*}
We will aim to find the saddle point of this function via primal-dual methods. As we have some prior knowledge of the 
optimal solutions, we will restrict the search space of each optimization variable to nicely chosen compact sets. For 
the $\bvec$ iterates, we consider the Euclidean ball domain $\bb{d}{D_{\bvec}} = 
\{\bvec\in\Rn^{d}~|~\twonorm{\bvec}\leq D_{\bvec}\}$ with the bound $D_{\bvec}>\|\Phim\transpose\muvec^{*}\|_{ 
\V{-2c}}$. Since the average reward of any policy is bounded in $[0,1]$, we naturally restrict the $\rho$ iterates 
to this domain. Finally, keeping in mind that Assumption~\ref{ass:boundedQ} guarantees that $\spannorm{\bm{q}^\pi} \le D_q$, 
we will also constrain the $\tvec$ iterates to an appropriate domain: $\bb{d}{D_{\tvec}} = 
\{\tvec\in\Rn^{d}~|~\twonorm{\tvec}\leq D_{\tvec}\}$. We will assume that this domain is large enough to represent all 
action-value functions, which implies that $D_{\tvec}$ should scale at least linearly with $D_q$. Indeed, we will 
suppose that the features are bounded as $\twonorm{\feat}\leq \Dphi$ for all $(x,a)\in\X\times\A$ so that our optimization 
algorithm only admits parametric $\bm{q}$ functions satisfying $\infnorm{\bm{q}}\leq \Dphi D_{\tvec}$. Obviously, 
$D_{\tvec}$ needs to be set large enough to ensure that it is possible at all to represent $\vec{q}$-functions with span $D_q$.

Thus, we aim to solve the following constrained optimization problem:
\begin{equation*}
	\min_{\rho\in[0,1], \vvec\in \Rn^{\X}, \tvec\in\bb{d}{D_{\tvec}}}\max_{\bvec\in\bb{d}{D_{\bvec}} , \muvec\in 
\Rn_{+}^{\X\times\A}} \Lag(\bvec, \muvec; \rho,\vvec,\tvec).
\end{equation*}
As done in the main text, we eliminate the high-dimensional variables $\vvec$ and $\muvec$ by committing to the choices
$\vvec=\vvec_{\tvec,\pi}$ and $\muvec=\muvec_{\bvec,\pi}$ defined as
\begin{align*}
	v_{\tvec,\pi}(x) &= \sum_{a}\pi(a|x)\iprod{\tvec}{\feat},\\
	\mu_{\bvec,\pi}(x,a) &= \pi(a|x)\ip{\psivec[x]}{\Vm^c\bvec}.
\end{align*}
This makes it possible to express the Lagrangian in terms of only $\bvec,\pi,\rho$ and $\tvec$:
\begin{align*}
	f(\bvec,\pi;\rho,\tvec)
	&= \add{\rho+}\langle\bvec\,, \V{c}  [\rtheta + \Psim\vvec_{\tvec,\pi} - \tvec \add{-\rho\vec{\varrho}}]\rangle + 
\langle \muvec_{\bvec,\pi}\,,\Phim\tvec - \Em\vvec_{\tvec,\pi} \rangle\\
	&= \add{\rho+}\langle\bvec\,, \V{c}  [\rtheta + \Psim\vvec_{\tvec,\pi} - \tvec \add{-\rho\vec{\varrho}}]\rangle
\end{align*}
The remaining low-dimensional variables $\bvec,\rho,\tvec$ are then updated using stochastic gradient descent/ascent. 
For this purpose it is useful to express the partial derivatives of the Lagrangian with respect to said variables:
\begin{align*}
	\vec{g}_{\bvec}
	&= \V{c}[\rtheta + \Psim\vvec_{\tvec,\pi} - \tvec \add{-\rho\vec{\varrho}}]\\
	g_{\rho}
	&= 1 - \iprod{\bvec}{\V{c}\vec{\varrho}}\\
	\vec{g}_{\tvec}
	&= \Phim\transpose \muvec_{\bvec,\pi} - \V{c}\bvec
\end{align*}

\subsection{Algorithm for average-reward MDPs}
Our algorithm for the AMDP setting has the same double-loop structure as the one for the discounted setting. In 
particular, the algorithm performs a sequence of outer updates $t=1,2,\dots,T$ on the policies $\pi_t$ and the 
\del{occupancy ratios}\add{iterates} $\bvec_t$, and then performs a sequence of updates $i=1,2,\dots,K$ in the inner loop to evaluate the 
policies and produce $\tvec_t$, $\rho_{t}$ and $\vvec_t$. Thanks to the reparametrization $\bvec=\V{-c}\lvec$, fixing $\pi_{t} = 
\text{softmax}(\sum_{k=1}^{t-1}\Phim\tvec_{k})$, $\vvec_{t}(x) = 
\sum_{a\in\A}\pi_{t}(a|x)\iprod{\varphi(x,a)}{\tvec_{t}}$ for $x\in\X$, and $\mu_{t}(x,a)= 
\pi_t(a|x)\iprod{\bm{\psi}(x)}{\V{c}\bvec_{t}}$ in round $t$ we can obtain unbiased estimates of the gradients of $f$ 
with respect to $\tvec$, $\bvec$, and $\rho$. For each primal update $t$, the algorithm uses a single sample transition
$(X_{t},A_{t},R_{t},X_{t}')$ generated by the behavior policy $\pi_{B}$ to compute an unbiased estimator of the first 
gradient $g_{\bvec}$ for that round as $\Tilde{\vec{g}}_{\bvec,t} =  \V{c-1}\phit[R_{t} + v_{t}(X_{t}') - 
\iprod{\tvec_{t}}{\phit}\add{-\rho_t}]$. Then, in iteration $i=1,\cdots,K$ of the inner loop within round $t$, we sample 
transitions $(X_{t,i},A_{t,i},R_{t,i},X_{t,i}')$ to compute gradient estimators with respect to $\rho$ and $\tvec$ as:
\begin{align*}
	\Tilde{g}_{\rho,t,i} &= 1 - \iprod{\phiti}{\V{c-1}\bvec_{t}}\\
	\Tilde{\vec{g}}_{\tvec,t,i} &= \phiti' \iprod{\phiti}{ \V{c-1}\bvec_{t}}  -  \phiti\iprod{\phiti}{\V{c-1}\bvec_{t}}.
\end{align*}
We have used the shorthand notation $\phiti = \varphi(X_{t,i},A_{t,i})$, $\phiti' = \varphi(X'_{t,i},A'_{t,i})$. The 
update steps are detailed in the pseudocode presented as Algorithm~\ref{alg:Main}.
\begin{algorithm}[t]
	\caption{Offline primal-dual method for Average-reward MDPs}
	\label{alg:Main}
	\begin{algorithmic}
		\STATE {\bfseries Input:}
		Learning rates $\zeta$, $\alpha$,$\xi$,$\eta$,
		initial iterates $\bvec_1\in\bb{d}{D_{\bvec}}$, $\rho_0\in[0,1]$, $\tvec_0\in\bb{d}{D_{\tvec}}$, $\pi_1\in\Pi$,
		\STATE
		\FOR{$t=1$ {\bfseries to} $T$}
		\STATE \emph{// Stochastic gradient descent}:
		\STATE Initialize: $\tvec_{t}^{(1)} = \tvec_{t-1}$;
		\FOR{$i=1$ {\bfseries to} $K$}
		\STATE Obtain sample $W_{t,i}=(X_{t,i},A_{t,i},R_{t,i},X'_{t,i})$;
		\STATE Sample $A_{t,i}'\sim \pi_{t}(\cdot|X_{t,i}')$;
		\STATE
		\STATE Compute $\Tilde{g}_{\rho, t, i} = 1 - \iprod{\phiti}{\V{c-1}\bvec_{t}}$;
		\STATE \qquad\quad\,\,\,\, $\Tilde{\vec{g}}_{\tvec, t, i} = \phiti' \iprod{\phiti}{ \V{c-1}\bvec_{t}}  -  \phiti\iprod{\phiti}{\V{c-1}\bvec_{t}}$;
		\STATE
		\STATE Update $\rho_{t}^{(i+1)} = \Pi_{[0,1]}(\rho_{t}^{(i)}- \xi \Tilde{g}_{\rho,t,i})$;
		\STATE \qquad\quad\,$\tvec_{t}^{(i+1)} = \Pi_{\bb{d}{D_{\tvec}}}(\tvec_{t}^{(i)} - \eta \Tilde{\vec{g}}_{\tvec,t,i})$.
		\ENDFOR
		\STATE Compute $\rho_{t} = \dfrac{1}{K}\sum_{i=1}^{K}\rho_{t}^{(i)}$;
		\STATE \qquad\quad\,\,\,\, $\tvec_{t} = \dfrac{1}{K}\sum_{i=1}^{K}\tvec_{t}^{(i)}$;
		
		\STATE
		\STATE \emph{// Stochastic gradient ascent}:
		\STATE Obtain sample $W_t = (X_t,A_t,R_t,X'_t)$;
		\STATE Compute $v_{t}(X_{t}') = \sum_{a} \pi_{t}(a|X_{t}')\iprod{\vec{\varphi}(X_{t}',a)}{\tvec_{t}}$;
		\STATE Compute $\Tilde{\vec{g}}_{\bvec,t} =  \V{c-1}\phit[R_{t} + v_{t}(X_{t}') - \iprod{\tvec_{t}}{\phit} - \rho_{t}]$;
		\STATE Update $\bvec_{t+1} = \Pi_{\bb{d}{D_{\bvec}}}(\bvec_{t} + \zeta \Tilde{\vec{g}}_{\bvec,t})$;
		\STATE
		\STATE \emph{// Policy update}:
		\STATE Compute $\pi_{t+1} = \softmax\pa{\alpha \sum_{k=1}^{t} \Phim\tvec_k}$.
		\ENDFOR
		\STATE {\bfseries Return:} $\pivec_J$ with $J\sim \Unif[T]$.
	\end{algorithmic}
\end{algorithm}

We now state the general form of our main result for this setting in Theorem \ref{thm:AMDP_main} below.
\begin{theorem}\label{thm:AMDP_main}
	Consider a linear MDP (Definition~\ref{def:linMDP}) such that $\tvec^{\pi}\in \bb{d}{D_{\tvec}}$ for all $\pi\in\Pi$. Further, suppose that $C_{\varphi,c}(\pi^*;\pi_B) \le D_{\bvec}$. Then, for any comparator policy $\pi^*\in\Pi$, the policy output by Algorithm~\ref{alg:Main} satisfies:
	\[
	\EE{\ip{\muvec^{\pi^*}-\muvec^{\piout}}{\rvec}}
	\leq \frac{2D_{\bvec}^{2}}{\zeta T}
	+ \frac{\log|\A|}{\alpha T}
	+ \frac{1}{2\xi K}
	+ \frac{2D_{\tvec}^{2}}{\eta K}
	+\frac{\zeta G_{\bvec,c}^2}{2}
	+\frac{\alpha D_{\tvec}^{2}D_{\fvec}^2}{2}
	+\frac{\xi G_{\rho,c}^2}{2}
	+\frac{\eta G_{\tvec,c}^2}{2},
	\]
	where
	\begin{align}
		&G_{\bvec,c}^2 = \Tr(\V{2c-1})(1+2D_{\tvec}\Dphi )^{2},\label{eq:amdpGbeta}\\
		&G_{\rho,c}^2 = 2\pa{1+ D_{\bvec}^2\norm{\Vm}_2^{2c-1}},\label{eq:amdpGrho}\\
		&G_{\tvec,c}^2 = 4 \Dphi^{2}D_{\bvec}^2\norm{\Vm}_2^{2c-1}.\label{eq:amdpGtheta}
	\end{align}
	 In particular, using learning rates $\zeta=\frac{2D_{\bvec}}{G_{\bvec,c}\sqrt{T}}$, 
$\alpha=\frac{\sqrt{2\log|\A|}}{D_{\tvec}D_{\fvec}\sqrt{T}}$,  $\xi=\frac{1}{G_{\rho,c}\sqrt{K}}$, and 
$\eta=\frac{2D_{\tvec}}{G_{\tvec,c}\sqrt{K}}$, and setting 
$K=T\cdot\frac{4D_{\bvec^2}G_{\bvec,c}^2+2D_{\tvec}^2D_{\fvec}^2\log|\A|}{G_{\rho,c}^2+4D_{\tvec}^2G_{\tvec,c}^2}$, we 
achieve $\EE{\ip{\muvec^{\pi^*}-\muvec^{\piout}}{\rvec}}\le \epsilon$ with a number of samples $n_\epsilon$ that is
	\begin{equation*}
		O\left({\epsilon^{-4}}{D_{\tvec}^4D_{\fvec}^4D_{\bvec}^4 \Tr(\V{2c-1})\norm{\Vm}_2^{2(2c-1)}\log|\A|}\right).
	\end{equation*}
\end{theorem}

By remark~\ref{remark1}, we have that $n_\epsilon$ is of order $O\left({\varepsilon^{-4}}{D_{\tvec}^4D_{\fvec}^{12c-2}D_{\bvec}^4 d^{2-2c} \log|\A|}\right)$.

\begin{corollary}\label{cavg}
	Assume that the bound of the feature vectors $D_{\fvec}$ is of order $O(1)$,
	that $D_{\rtheta}=D_{\bm{\psi}}=\sqrt{d}$ which together imply $D_{\tvec} \le \sqrt{d} + 1 + \sqrt{d}D_{q} = O(\sqrt{d}D_q)$ and that
	$D_{\bvec} = c\cdot C_{\varphi,c}(\pi^*;\pi_B)$ for some positive universal constant $c$. Then, under the same
	assumptions of Theorem \ref{thm:main_disc}, $n_{\varepsilon}$ is of order $O\left({\varepsilon^{-4}}{D_{q}^4C_{\varphi,c}(\pi^*;\pi_B)^2 d^{4-2c} \log|\A|}\right)$.
\end{corollary}
Recall that $C_{\varphi,1/2}$ is always smaller than $C_{\varphi,1}$, but using $c=1/2$ in the algorithm requires 
knowledge of the covariance matrix $\Vm$, and results in a 
slightly worse dependence on the dimension.

The proof of Theorem~\ref{thm:AMDP_main} mainly follows the same steps as in the discounted case, with some 
added difficulty that is inherent in the more challenging average-reward setup. Some key challenges include treating 
the 
additional optimization variable $\rho$ and coping with the fact that the optimal parameters $\tvec^*$ and $\bvec^*$ 
are 
not necessarily unique any more.

\subsection{Analysis}\label{sec:AMDP_analysis}
We now prove our main result regarding the AMDP setting in Theorem \ref{thm:AMDP_main}. Following the derivations in the main 
text, we study the dynamic duality gap defined as
\begin{align}
	&\GG_T(\bvec^*,\pi^*;\rho^{*}_{1:T},\tvec^*_{1:T}) =
	\frac{1}{T}\sum_{t=1}^T \bpa{f(\bvec^{*},\pi^{*};\rho_t,\tvec_t) - f(\bvec_t,\pi_t;\rho_t^{*},\tvec^{*}_t)}.
\end{align}
First we show in Lemma \ref{lem:duality-to-suboptimality} below that, for appropriately chosen comparator points, the 
expected suboptimality of the policy returned by Algorithm~\ref{alg:Main} can be upper bounded in terms of the
expected dynamic duality gap.
\begin{lemma}\label{lem:duality-to-suboptimality}
	Let $\tvec^*_t$ such that  $\iprod{\feat}{\tvec^{*}_{t}} = \iprod{\feat}{\tvec^{\pi_{t}}} - 
\inf_{(x,a)\in\X\times\A}\iprod{\feat}{\tvec^{\pi_{t}}}$ holds for all $(x,a)\in\X\times\A$, and let $\vvec^*_t$ be defined as 
$\vvec^*_t(x) = \sum_{a\in\A}\pi_{t}(a|x)\iprod{\feat}{\tvec_{t}^{*}}$ for all $x$. Also, let $\rho_{t}^{*} = 
\rho^{\pi_{t}}$, $\pi^*$ be an optimal policy, and $\bvec^*=\Vm^{-c}\Phim^\top\muvec^{*}$ where $\muvec^{*}$ is 
the occupancy measure of $\pi^*$. Then, the suboptimality gap of the policy output by Algorithm~\ref{alg:Main} satisfies
	\[\EET{\ip{\muvec^{*}-\muvec^{\piout}}{\rvec}} = \GG_T(\bvec^*,\pi^*;\rho^{*}_{1:T},\tvec^*_{1:T}).\]
\end{lemma}

\begin{proof}
Substituting $(\bvec^{*},\pi^{*}) = ( \V{-c}\Phim\transpose \muvec^{*}, \pi^{*})$ in the first term of the dynamic 
duality gap we have
	\begin{align*}
		f(\bvec^{*},\pi^{*};\rho_t,\tvec_t)
		&\,\,= \add{\rho_{t} + }\langle \V{-c}\Phim\transpose \muvec^{*}\,, \V{c}  [\rtheta + 
\Psim\vvec_{\tvec_t,\pi^{*}} - \tvec_{t}\add{-\rho_{t}\vec{\varrho}}]\rangle\\
		&\,\,= \add{\rho_{t} + }\langle \muvec^{*}\,,r + \Pm\vvec_{\tvec_t,\pi^{*}} - \Phim\tvec_{t}\add{-\rho_{t}\vec{1}}\rangle\\
		&\,\,= \langle \muvec^{*}\,,r\rangle + \langle \muvec^{*}\,,\Em\vvec_{\tvec_t,\pi^{*}} - \Phim\tvec_{t}\rangle + \rho_{t}[1-\iprod{\muvec^{*}}{\vec{1}}]\\
		&\,\,= \langle \muvec^{*}\,,r\rangle.
	\end{align*}
	Here, we have used the fact that $\muvec^{*}$ is a valid occupancy measure, so it satisfies the flow constraint 
$\Em\transpose\muvec^{*}=\Pm\transpose\muvec^{*}$ and the normalization constraint $\iprod{\muvec^{*}}{\vec{1}}=1$. 
Also, in the last step we have used the definition of $\vvec_{\theta_t,\pi^*}$ that guarantees that the following 
equality holds:
	\begin{equation*}
		\iprod{\muvec^{*}}{\Phim\tvec_{t}}
		= \sum_{x\in\X}\nu^{*}(x)\sum_{a\in\A}\pi^{*}(a|x)\iprod{\tvec_t}{\feat}
		= \sum_{x\in\X}\nu^{*}(x)v_{\tvec_t,\pi^{*}}(x)
		= \langle \muvec^{*}\,,\Em\vvec_{\tvec_t,\pi^{*}}\rangle.
	\end{equation*} 
	 For the second term in the dynamic duality gap, using that $\pi_{t}$ is $\F_{t-1}$-measurable we write
	\begin{align*}
		&f(\bvec_t,\pi_t;\rho_t^{*},\tvec^{*}_t)\\
		&\quad=\add{\rho^{*}_t+}\langle\bvec_t\,, \V{c}  [\rtheta + \Psim\vvec_{\tvec^{*}_t,\pi_t} - \tvec^{*}_t 
		\add{-\rho^{*}_t\vec{\varrho}}]\rangle\\
		&\quad=\add{\rho^{*}_t+}\langle\bvec_t\,, \V{c-1} \EEt{\fvec_t\fvec_t\transpose[\rtheta + 
			\Psim\vvec_{\tvec^{*}_t,\pi_t} - \tvec^{*}_t \add{-\rho^{*}_t\vec{\varrho}}]}\rangle\\
		&\quad=\rho^{*}_t+\iprod{\bvec_t}{\EEt{\V{c-1}  \phit\bigg[R_t + \sum_{x,a}p(x|X_t,A_t)\pi_{t}(a|x)\iprod{\feat} 
				{\tvec_{t}^{*}} - \iprod{\fvec(X_t,A_t)}{\tvec_{t}^{*}} -\rho^{*}_t\bigg]}}\\
		&\quad=\rho^{\pi_{t}}+\iprod{\bvec_t}{\EEt{\V{c-1}  \phit\bigg[R_t + \sum_{x,a}p(x|X_t,A_t)\pi_{t}(a|x)\iprod{\feat} 
				{\tvec^{\pi_{t}}} - \iprod{\fvec(X_t,A_t)}{\tvec^{\pi_{t}}} -\rho^{\pi_{t}}\bigg]}}\\
		&\quad=\rho^{\pi_{t}}
		+ \langle\bvec_{t}\,,\EEt{\V{c-1} \phit [r(X_t,A_t) + \iprod{p(\cdot|X_{t},A_{t})}{v^{\pi_{t}}} - 
			q^{\pi_t}(X_t,A_t) \add{-\rho^{\pi_{t}}}]}\rangle\\
		&\quad= \rho^{\pi_{t}} = \iprod{\muvec^{\pi_{t}}}{r},
	\end{align*}
where in the fourth equality we used that $\iprod{\fvec(x,a) - \fvec(x',a')}{\theta^*_t} = \iprod{\fvec(x,a) - 
\fvec(x',a')}{\theta^{\pi_t}}$ holds for all $x,a,x',a'$ by definition of $\theta_t^*$. Then, the last equality follows from the fact that the Bellman equations for $\pi_t$ imply $q^{\pi_{t}}(x,a) + \rho^{\pi_{t}}= r(x,a) + \iprod{p(\cdot|x,a)}{\vvec^{\pi_{t}}}$.
	
	Combining both expressions for $f(\bvec^{*},\pi^{*};\rho_t,\tvec_t)$ and $f(\bvec_t,\pi_t;\rho_t^{*},\tvec^{*}_t)$ in the dynamic duality gap we have:
	\begin{align*}
		\GG_T(\bvec^*,\pi^*;\rho^{*}_{1:T},\tvec^*_{1:T})
		= 
		\frac{1}{T}\sum_{t=1}^{T} \bpa{\iprod{\muvec^{*} - \muvec^{\pi_{t}}}{r}\del{
				- \rho(\pi_{t})[\langle\bvec_{t}\,, \V{} \vec{\varrho}\rangle - 1]}}
		=
		\EET{\iprod{\muvec^{*} - \muvec^{\pi_{\text{out}}}}{r}}.
	\end{align*}
	The second equality follows from noticing that, since $\piout$ is sampled uniformly from $\{\pi_t\}_{t=1}^T$, 
	$\EE{\ip{\muvec^{\piout}}{\rvec}} = \frac{1}{T}\sum_{t=1}^T\EE{\ip{\muvec^{\pi_t}}{\rvec}}$. This completes the proof.
\end{proof}
Having shown that for well-chosen comparator points the dynamic duality gap equals the expected suboptimality of the 
output policy of Algorithm~\ref{alg:Main}, it remains to relate the gap to the optimization error of the primal-dual 
procedure. This is achieved in the following lemma.
\begin{lemma}\label{lem:duality-to-opterror}
 For the same choice of comparators $(\bvec^*,\pi^*;\rho^{*}_{1:T},\tvec^*_{1:T})$ as in 
Lemma~\ref{lem:duality-to-suboptimality} the dynamic duality gap associated with the iterates produced by 
Algorithm~\ref{alg:Main} satisfies
	\begin{align*}
		&\EE{\GG_T(\bvec^*,\pi^*;\rho^{*}_{1:T},\tvec^*_{1:T})}\\
		&\qquad\leq \frac{2D_{\bvec}^{2}}{\zeta T}
		+ \frac{\HHKL{\pi^*}{\pi_1}}{\alpha T}
		+ \frac{1}{2\xi K}
		+ \frac{2D_{\tvec}^{2}}{\eta K}\\
		&\qquad\quad +\frac{\zeta \Tr(\V{2c-1})(1+2\Dphi D_{\tvec})^{2}}{2}
		+\frac{\alpha \Dphi^2D_{\tvec}^{2}}{2}
		+\xi \pa{1+ D_{\bvec}^2\norm{\Vm}_2^{2c-1}}
		+2\eta \Dphi^{2}D_{\bvec}^2\norm{\Vm}_2^{2c-1}.
	\end{align*}
\end{lemma}
\begin{proof}
	The first part of the proof follows from recognising that the dynamic duality gap can be rewritten in terms of the 
total regret of the primal and dual players in the algorithm. Formally, we write 
\begin{align*}
	&\GG_T(\bvec^*,\pi^*;\rho^{*}_{1:T},\tvec^*_{1:T})\\
	&\quad= 
	\frac{1}{T}\sum_{t=1}^{T} \pa{f(\bvec^{*},\pi^{*};\rho_t,\tvec_t) - f(\bvec_t,\pi_t;\rho_t,\tvec_t)}
	+\frac{1}{T}\sum_{t=1}^{T} \pa{f(\bvec_t,\pi_t;\rho_t,\tvec_t) - f(\bvec_t,\pi_t;\rho_t^{*},\tvec^{*}_t)}.
	\end{align*}
	Using that $\bvec^{*} = \Vm^{-c}\Phim^\top\muvec^{*}, \vec{q}_{t} = \iprod{\feat}{\tvec_t}$, $\vvec_t = 
\vvec_{\tvec_t,\pi_t}$ and that  $\vec{g}_{\bvec,t}=\V{c}  [\rtheta + \Psim\vvec_t - 
\tvec_{t}\add{-\rho_{t}\vec{\varrho}}]$, we see that term in the first sum can be simply rewritten as
	\begin{align*}
	&f(\bvec^{*},\pi^{*};\rho_t,\tvec_t) - f(\bvec_t,\pi_t;\rho_t,\tvec_t)\\
		&\qquad\quad= \langle \bvec^{*}\,, \V{c}  [\rtheta + \Psim\vvec_{\tvec_t,\pi^{*}} - 
\tvec_{t}\add{-\rho_{t}\vec{\varrho}}]\rangle - \langle \bvec_t\,, \V{c}  [\rtheta + \Psim\vvec_{\tvec_t,\pi_t} - 
\tvec_{t}\add{-\rho_{t}\vec{\varrho}}]\rangle\\
		&\qquad\quad= \langle \bvec^{*} - \bvec_t\,, \V{c}  [\rtheta + \Psim\vvec_t - \tvec_{t}\add{-\rho_{t}\vec{\varrho}}]\rangle
		+ \langle \Psim\transpose\V{c}\bvec^{*}\,,  \vvec_{\tvec_t,\pi^{*}} - \vvec_{\tvec_t,\pi_t}\rangle\\
		&\qquad\quad= \langle \bvec^{*} - \bvec_t\,, \vec{g}_{\bvec,t}\rangle
		+ \sum_{x\in\X}\nu^{*}(x)\iprod{\pi^{*}(\cdot|x) - \pi_{t}(\cdot|x)}{\vec{q}_{t}(x,\cdot)}.
	\end{align*}
	In a similar way, using that $\Em\transpose\muvec_t=\Psim\transpose\V{c}\bvec_t$ and the definitions of the 
gradients $g_{\rho,t}$ and $\vec{g}_{\tvec,t}$, the term in the second sum can be rewritten as
	\begin{align*}
		&f(\bvec_t,\pi_t;\rho_t,\tvec_t) - f(\bvec_t,\pi_t;\rho_t^{*},\tvec^{*}_t)\\
		&\qquad\quad= \rho_{t} + \langle \bvec_t\,, \V{c}  [\rtheta + \Psim\vvec_{\tvec_t,\pi_t} - 
\tvec_{t}\add{-\rho_{t}\vec{\varrho}}]\rangle - \rho_{t}^{*}- \langle \bvec_t\,, \V{c}  [\rtheta + 
\Psim\vvec_{\tvec_t^{*},\pi_t} - \tvec_{t}^{*}\add{-\rho_{t}^{*}\vec{\varrho}}]\rangle\\
		&\qquad\quad= (\rho_{t}- \rho_{t}^{*})[1 - \iprod{\bvec_{t}}{\V{c}\vec{\varrho}}]
		-\iprod{\tvec_t - \tvec_{t}^{*}}{\V{c}\bvec_t}
		+ \iprod{\Em\transpose\muvec_t}{\vvec_{\tvec_t,\pi_t} - \vvec_{\tvec_t^{*},\pi_t}}\\
		&\qquad\quad= (\rho_{t}- \rho_{t}^{*})[1 - \iprod{\bvec_{t}}{\V{c}\vec{\varrho}}]
		-\iprod{\tvec_t - \tvec_{t}^{*}}{\V{c}\bvec_t}
		+ \iprod{\Phim\transpose\muvec_t}{\tvec_t - \tvec_{t}^{*}}\\
		&\qquad\quad= (\rho_{t}- \rho_{t}^{*})[1 - \iprod{\bvec_{t}}{\V{c}\vec{\varrho}}]
		+\iprod{\tvec_t - \tvec_{t}^{*}}{\Phim\transpose\muvec_t - \V{c}\bvec_t}\\
		&\qquad\quad= (\rho_{t}- \rho_{t}^{*})g_{\rho,t}
		+\iprod{\tvec_t - \tvec_{t}^{*}}{\vec{g}_{\tvec,t}} = \frac{1}{K}\sum_{i=1}^{K}\pa{(\rho_{t}^{(i)}- \rho_{t}^{*})g_{\rho,t}
		+\iprod{\tvec_t^{(i)} - \tvec_{t}^{*}}{\vec{g}_{\tvec,t}}}.
	\end{align*}
Combining both terms in the duality gap concludes the first part of the proof. As shown below the dynamic duality 
gap is written as the error between iterates of the algorithm from respective comparator points in the direction of the 
exact gradients. Formally, we have
	\begin{align*}
		\GG_T(\bvec^*,\pi^*;\rho^{*}_{1:T},\tvec^*_{1:T})
		&= 
		\frac{1}{T}\sum_{t=1}^{T} \pa{\langle \bvec^{*} - \bvec_t\,, \vec{g}_{\bvec,t}\rangle
		+ \sum_{x\in\X}\nu^{*}(x)\iprod{\pi^{*}(\cdot|x) - \pi_{t}(\cdot|x)}{\vec{q}_{t}(x,\cdot)}}\\
		&\quad+\frac{1}{TK}\sum_{t=1}^{T}\sum_{i=1}^{K}\pa{(\rho_{t}^{(i)}- \rho_{t}^{*})g_{\rho,t}
			+\iprod{\tvec_t^{(i)} - \tvec_{t}^{*}}{\vec{g}_{\tvec,t}}}.
	\end{align*}
	Then, implementing techniques from stochastic gradient descent analysis in the proof of Lemmas \ref{lem:gbeta} to 
\ref{lem:gtheta} and mirror descent analysis in Lemma \ref{lem:regret-pi-disc}, the expected dynamic duality gap can be 
upper bounded as follows:
	\begin{align*}
		&\EE{\GG_T(\bvec^*,\pi^*;\rho^{*}_{1:T},\tvec^*_{1:T})}\\
		&\qquad\leq \frac{2D_{\bvec}^{2}}{\zeta T}
		+ \frac{\HHKL{\pi^*}{\pi_1}}{\alpha T}
		+ \frac{1}{2\xi K}
		+ \frac{2D_{\tvec}^{2}}{\eta K}\\
		&\qquad\quad +\frac{\zeta \Tr(\V{2c-1})(1+2\Dphi D_{\tvec})^{2}}{2}
		+\frac{\alpha \Dphi^2D_{\tvec}^{2}}{2}
		+\xi \pa{1+ D_{\bvec}^2\norm{\Vm}_2^{2c-1}}
		+2\eta \Dphi^{2}D_{\bvec}^2\norm{\Vm}_2^{2c-1}.
	\end{align*}
	This completes the proof
\end{proof}
\paragraph{Proof of Theorem \ref{thm:AMDP_main}}
First, we bound the expected suboptimality gap by combining Lemma \ref{lem:duality-to-suboptimality} and 
\ref{lem:duality-to-opterror}. Next, bearing in mind that the algorithm only needs $T(K+1)$ total samples from the 
behavior policy we optimize the learning rates to obtain a bound on the sample complexity, thus completing the proof.
\qed

\subsection{Missing proofs for Lemma \ref{lem:duality-to-opterror}}\label{app:prelim}
In this section we prove Lemmas~\ref{lem:gbeta} to~\ref{lem:gtheta} used in the proof of Lemma \ref{lem:duality-to-opterror}. It is important to recall that sample transitions $(X_{k},A_{k},R_{t},X'_{k})$ in any iteration $k$ are generated 
in the following way: we draw i.i.d state-action pairs $(X_{k},A_{k})$ from $\muvec_{B}$, and for each state-action 
pair, the next $X'_{k}$ is sampled from $p(\cdot|X_{k},A_{k})$ and immediate reward computed as $R_{t} = 
r(X_{k},A_{k})$. Precisely in iteration $i$ of round $t$ where $k=({t,i})$, since $(X_{t,i},A_{t,i})$ are sampled i.i.d 
from $\muvec_{B}$ at this time step, $\EEti{\phiti\phiti\transpose} = \EEs{\feat\feat\transpose}{(x,a)\sim\muvec_{B}} =  
\V{} $.

\begin{lemma}\label{lem:gbeta}
	The gradient estimator 
	$\Tilde{\vec{g}}_{\bvec,t}$ satisfies 
$\EEc{\Tilde{\vec{g}}_{\bvec,t}}{\F_{t-1},\tvec_{t}} = \vec{g}_{\bvec,t}$ and
	\begin{equation*}
		\EE{\|\Tilde{\vec{g}}_{\bvec,t}\|_{2}^{2}}\leq \Tr(\V{2c-1})(1+2\Dphi D_{\tvec})^{2}.
	\end{equation*}
	Furthermore, for any $\bvec^{*}$ with $\bvec^* \in \bb{d}{D_{\bvec}}$, the iterates $\bvec_t$ satisfy
	\begin{equation}\label{eqn:betabound}
		\EE{\sum_{t=1}^{T} \langle\bvec^{*} - \bvec_{t}\,,\vec{g}_{\bvec,t}\rangle} \leq \frac{2D_{\bvec}^{2}}{\zeta} + 
\frac{\zeta T\Tr(\V{2c-1})(1+2\Dphi D_{\tvec})^{2}}{2}.
	\end{equation}
\end{lemma}
\begin{proof}
For the first part, we remind that $\pi_{t}$ is $\F_{t-1}$-measurable and $\vvec_{t}$ is determined given $\pi_{t}$ and 
$\tvec_{t}$. Then, we write
	\begin{align*}
		\EEc{\Tilde{\vec{g}}_{\bvec,t}}{\F_{t-1},\tvec_{t}}
		&= \EEc{ \V{c-1}\phit[R_{t} + v_{t}(X_{t}') - \iprod{\tvec_{t} }{\phit}- \rho_{t}]}{\F_{t-1},\tvec_{t}}\\
		&= \EEc{ \V{c-1}\phit[R_{t} + \EEs{v_{t}(x')}{x'\sim p(\cdot|X_{t},A_{t})} - \iprod{\tvec_{t}}{\phit}-\rho_{t}]}{\F_{t-1},\tvec_{t}}\\
		&= \EEc{ \V{c-1}\phit[R_{t} + \iprod{p(\cdot|X_{t},A_{t})}{\vvec_{t}} - \iprod{\tvec_{t}}{\phit} - \rho_{t}]}{\F_{t-1},\tvec_{t}}\\
		&= \EEc{ \V{c-1}\phit\phit\transpose[\rtheta + \Psim\vvec_{t} - \tvec_{t} - 
\rho_{t}\vec{\varrho}]}{\F_{t-1},\tvec_{t}}\\
		&= \V{c-1}\EEc{ \phit\phit\transpose}{\F_{t-1},\tvec_{t}}[\rtheta + \Psim\vvec_{t} - \tvec_{t} - 
\rho_{t}\vec{\varrho}]\\
		&= \V{c} [\rtheta + \Psim\vvec_{t} - \tvec_{t} - \rho_{t}\vec{\varrho}]
		= \vec{g}_{\bvec,t}.
	\end{align*}
Next, we use the facts that $r\in[0,1]$ and $\|\vvec_{t}\|_{\infty}\leq \|\Phim\tvec_{t}\|_{\infty}\leq \Dphi 
D_{\tvec}$ to show the following bound:
	\begin{align*}
		\EEc{\|\Tilde{\vec{g}}_{\bvec,t}\|_{2}^{2}}{\F_{t-1},\tvec_{t}}
		&=\EEc{\sqtwonorm{ \V{c-1}\phit[R_{t} + v_{t}(X_{t}') - \iprod{\tvec_{t}}{\phit}]}}{\F_{t-1},\tvec_{t}}\\
		&=\EEc{\abs{R_{t} + v_{t}(X_{t}') - \iprod{\tvec_{t}}{\phit}}\sqtwonorm{ \V{c-1}\phit}}{\F_{t-1},\tvec_{t}}\\
		&\leq\EEc{(1+2\Dphi D_{\tvec})^{2}\sqtwonorm{ \V{c-1}\phit}}{\F_{t-1},\tvec_{t}}\\
		&=(1+2\Dphi D_{\tvec})^{2}\EEc{\phit\transpose\V{2(c-1)}\phit}{\F_{t-1},\tvec_{t}}\\
		&=(1+2\Dphi D_{\tvec})^{2}\EEc{\Tr(\V{2(c-1)}\phit\phit\transpose)}{\F_{t-1},\tvec_{t}}\\
		&\leq \Tr(\V{2c-1})(1+2\Dphi D_{\tvec})^{2}.
	\end{align*}
	The last step follows from the fact that $\V{}$, hence also $\V{2c-1}$, is positive semi-definite, so 
$\Tr(\V{2c-1})\ge 0$. Having shown these properties, we appeal to the standard analysis of online gradient descent 
stated as Lemma~\ref{lem:aux-sgd} to obtain the following bound
	\begin{equation*}
		\EE{\sum_{t=1}^{T} \langle\bvec^{*} - \bvec_{t}\,,\vec{g}_{\bvec,t}\rangle} \leq 
\frac{\sqtwonorm{\bvec_{1}-\bvec^{*}}}{2\zeta} + \frac{\zeta T\Tr(\V{2c-1})(1+2\Dphi D_{\tvec})^{2}}{2}.
	\end{equation*}
	Using that $\twonorm{\bvec^{*}} \le D_{\bvec}$ concludes the proof.
\end{proof}

\begin{lemma}\label{lem:grho}
	The gradient estimator $\Tilde{g}_{\rho,t,i}$ satisfies $\EEti{\Tilde{g}_{\rho,t,i}} = g_{\rho,t}$ and 
$\EEti{\Tilde{g}_{\rho,t,i}^{2}}\leq 2+ 2D_{\bvec}^2\norm{\Vm}_2^{2c-1}$. Furthermore, for any $\rho^{*}_{t}\in[0,1]$, the iterates $\rho_{t}^{(i)}$ satisfy
	\begin{equation*}
		\EE{\sum_{i=1}^{K}(\rho_{t}^{(i)}-\rho_{t}^{*})g_{\rho,t}}
		\leq \frac{1}{2\xi} + \xi K\pa{1+ \|\bvec_{t}\|_{\V{2c-1}}^{2}}.
	\end{equation*}
\end{lemma}
\begin{proof}
	For the first part of the proof, we use that $\bvec_{t}$ is $\F_{t,i-1}$-measurable, to obtain
	\begin{align*}
		\EEti{\Tilde{g}_{\rho,t,i}}
		&=\EEti{1-\iprod{\phiti}{\V{c-1}\bvec_{t}}}\\
		&=\EEti{1-\iprod{\phiti\phiti\transpose\vec{\varrho}}{\V{c-1}\bvec_{t}}}\\
		&= 1-\iprod{\V{c}\vec{\varrho}}{\bvec_{t}} = g_{\rho,t}.
	\end{align*}
In addition, using Young's inequality and  $\norm{\bvec_t}^2_{\Vm^{2c-1}}\le D_{\bvec}^2\norm{\Vm}_2^{2c-1}$ we show that
	\begin{align*}
		\EEti{\Tilde{g}_{\rho,t,i}^{2}}
		&= \EEti{\sq{1-\iprod{\phiti}{\V{c-1}\bvec_{t}}}}\\
		&\leq 2 + 2\EEti{\bvec_{t}\transpose\V{c-1}\phiti\phiti\transpose\V{c-1}\bvec_{t}}\\
		&=2+ 2\|\bvec_{t}\|_{\V{2c-1}}^{2}
		\leq 2+ 2D_{\bvec}^2\norm{\Vm}_2^{2c-1}.
	\end{align*}
	For the second part, we appeal to the standard online gradient descent analysis of Lemma~\ref{lem:aux-sgd} to bound on 
the total error of the iterates:
		\begin{equation*}
		\EE{\sum_{i=1}^{K}(\rho_{t}^{(i)}-\rho_{t}^{*})g_{\rho,t}}
		\leq \frac{\sq{\rho^{(1)}_{t} - \rho^{*}_{t}}}{2\xi} + \xi K\pa{1+ D_{\bvec}^2\norm{\Vm}_2^{2c-1}}.
	\end{equation*}
	Using that $\bpa{\rho^{(1)}_{t} - \rho^{*}_{t}}^2\leq 1$ concludes the proof. 
\end{proof}

\begin{lemma}\label{lem:gtheta}
	The gradient estimator $\Tilde{\vec{g}}_{\tvec,t,i}$ satisfies $\EEti{\Tilde{\vec{g}}_{\tvec,t,i}} = \vec{g}_{\tvec,t,i}$ and $\EEti{\|\Tilde{\vec{g}}_{\tvec,t,i}\|_{2}^{2}}\leq 4\Dphi^{2}D_{\bvec}^2\norm{\Vm}_2^{2c-1}$. Furthermore, for any $\tvec_{t}^{*}$ with $\twonorm{\tvec_t^*}\le D_{\tvec}$, the 
iterates $\tvec_{t}^{(i)}$ satisfy
	\begin{equation}\label{eqn:thetabound}
		\EE{\sum_{i=1}^{K}\iprod{\tvec_{t}^{(i)} - \tvec_{t}^{*}}{\vec{g}_{\tvec,t,i}}}
		\leq \frac{2D_{\tvec}^{2}}{\eta} + 2\eta K\Dphi^{2}D_{\bvec}^2\norm{\Vm}_2^{2c-1}.
	\end{equation}
\end{lemma}
\begin{proof}
	Since $\bvec_{t},\pi_{t},\rho_{t}^{i}$ and $\tvec_{t}^{i}$ are $\F_{t,i-1}$-measurable, we obtain
	\begin{align*}
		\EEti{\Tilde{\vec{g}}_{\tvec,t,i}}
		&=\EEti{\phiti' \iprod{\phiti}{ \V{c-1}\bvec_{t}}  -  \phiti\iprod{\phiti}{\V{c-1}\bvec_{t}}}\\
		&=\Phim\transpose\EEti{\vec{e}_{X'_{t,i},A_{t,i}'}\iprod{\phiti}{ \V{c-1}\bvec_{t}}}  -  \EEti{\phiti\phiti\transpose}\V{c-1}\bvec_{t}\\
		&=\Phim\transpose\EEti{[\pi_{t}\circ p(\cdot|X_{t},A_{t})]\iprod{\phiti}{ \V{c-1}\bvec_{t}}}  -  \V{c} \bvec_{t}\\
		&=\Phim[\pi_{t}\circ \Psim\transpose\EEti{\phiti\phiti\transpose} \V{c-1}\bvec_{t}]  -  \V{c} \bvec_{t}\\
		&=\Phim[\pi_{t}\circ \Psim\transpose \V{c} \bvec_{t}]  -  \V{c} \bvec_{t}\\
		&= \Phim\transpose \muvec_{t} -  \V{c} \bvec_{t} = \vec{g}_{\tvec,t}.
	\end{align*}
	Next, we consider the squared gradient norm and bound it via elementary manipulations as follows:
	\begin{align*}
		\EEti{\sqtwonorm{\Tilde{\vec{g}}_{\tvec,t,i}}}
		&= \EEti{\sqtwonorm{\phiti' \iprod{\phiti}{ \V{c-1}\bvec_{t}}  -  \phiti\iprod{\phiti}{\V{c-1}\bvec_{t}}}}\\
		&\leq 2\EEti{\sqtwonorm{\phiti' \iprod{\phiti}{ \V{c-1}\bvec_{t}}}} + 2\EEti{\sqtwonorm{ \phiti\iprod{\phiti}{\V{c-1}\bvec_{t}}}}\\
		&= 2\EEti{\bvec_{t}\transpose\V{c-1} \phiti\sqtwonorm{\phiti'}\phiti\transpose \V{c-1}\bvec_{t}}
		+ 2\EEti{\bvec_{t}\transpose\V{c-1} \phiti\sqtwonorm{\phiti}\phiti\transpose \V{c-1}\bvec_{t}}\\
		&\leq  2\Dphi^{2} \EEti{\bvec_{t}\transpose\V{c-1} \phiti\phiti\transpose \V{c-1}\bvec_{t}}
		+ 2\Dphi^{2}\EEti{\bvec_{t}\transpose\V{c-1} \phiti\phiti\transpose \V{c-1}\bvec_{t}}\\
		&=  2\Dphi^{2} \EEti{\bvec_{t}\transpose\V{c-1} \Vm \V{c-1}\bvec_{t}}
		+ 2\Dphi^{2}\EEti{\bvec_{t}\transpose\V{c-1} \Vm \V{c-1}\bvec_{t}}\\
		&\leq 4\Dphi^{2}\|\bvec_{t}\|_{\V{2c-1}}^{2}\leq 4\Dphi^{2}D_{\bvec}^2\norm{\Vm}_2^{2c-1}.
	\end{align*}
	Having verified these conditions, we appeal to the online gradient descent analysis of Lemma~\ref{lem:aux-sgd} to 
show the bound
	\begin{equation*}
		\EE{\sum_{i=1}^{K}\iprod{\tvec_{t}^{(i)} - \tvec_{t}^{*}}{\vec{g}_{\tvec,t}}}
		\leq \frac{\sqtwonorm{\tvec^{(1)}_{t} - \tvec^{*}_{t}}}{2\eta} + 2\eta K\Dphi^{2}D_{\bvec}^2\norm{\Vm}_2^{2c-1}.
	\end{equation*}
	We then use that $\twonorm{\tvec_{t}^{*} - \tvec_{t}^{(1)}}\leq 2D_{\tvec}$ for 
$\tvec_{t}^{*},\tvec_{t}^{(1)}\in\bb{d}{D_{\tvec}}$, thus concluding the proof.
\end{proof}

	\newpage
	\section{Auxiliary Lemmas}
The following is a standard result in convex optimization proved here for the sake of completeness---we refer to 
\citet{NY83,Zin03,Ora19} for more details and comments on the history of this result. 
\begin{lemma}[Online Stochastic Gradient Descent]\label{lem:aux-sgd}
	Given $y_{1}\in\bb{d}{D_{y}}$ and $\eta>0$, define the sequences $y_{2},\cdots,y_{n+1}$ and $h_1,\cdots,h_n$ such 
that for $k=1,\cdots,n$,
	\begin{equation*}
	y_{k+1} = \Pi_{\bb{d}{D_{y}}}\pa{y_{k} + \eta \wh{h}_{k}},
	\end{equation*}
	and $\wh{h}_{k}$ satisfies $\EEc{\wh{h}_{k}}{\F_{k-1}}=h_k$ and $\EEc{\sqtwonorm{\wh{h}_{k}}}{\F_{k-1}}\leq G^2$. 
Then, for $y^{*}\in\bb{d}{D_{y}}$:
	\begin{equation*}
	\EE{\sum_{k=1}^{n} \ip{y^{*} - y_{k}}{h_{k}}} \leq \frac{\sqtwonorm{y_{1} - y^{*}}}{2\eta} + \frac{\eta nG^2}{2}.
	\end{equation*}
\end{lemma}
\begin{proof}
	We start by studying the following term:
	\begin{align*}
	\sqtwonorm{y_{k+1} - y^{*}}
	&= \sqtwonorm{\Pi_{\bb{d}{D_{y}}}(y_{k} + \eta \wh{h}_{k}) - y^{*}}\\
	&\leq \sqtwonorm{y_{k} + \eta \wh{h}_{k} - y^{*}}\\
	&= \sqtwonorm{y_{k} - y^{*}} - 2\eta\iprod{y^{*} - y_{k}}{\wh{h}_{k}} + \eta^{2}\sqtwonorm{\wh{h}_{k}}.
	\end{align*}
	The inequality is due to the fact that the projection operator is a non-expansion with respect to the Euclidean 
norm. Since $\EEc{\wh{h}_{k}}{\F_{k-1}} =h_{k}$, we can rearrange the above equation and take a conditional expectation 
to obtain
	\begin{align*}
	\iprod{y^{*} - y_{k}}{h_{k}}
	&\leq \frac{\sqtwonorm{y_{k}  - y^{*}} - \EEc{\sqtwonorm{y_{k+1} - y^{*}}}{\F_{k-1}}}{2\eta} + 
\frac{\eta}{2}\EEc{\sqtwonorm{\wh{h}_{k}}}{\F_{k-1}} \\
	&\le\frac{\sqtwonorm{y_{k}  - y^{*}} - \EEc{\sqtwonorm{y_{k+1} - y^{*}}}{\F_{k-1}}}{2\eta} + \frac{\eta G^2}{2},
	\end{align*}
	where the last inequality is from $\EEc{\sqtwonorm{\wh{h}_{k}}}{\F_{k-1}}\leq G^2$.
	Finally, taking a sum over $k=1,\cdots,n$, taking a marginal expectation, evaluating the resulting telescoping sum  
and upper-bounding negative terms by zero we obtain the desired result as
	\begin{align*}
	\EE{\sum_{k=1}^{n}\iprod{y^{*} - y_{k}}{\hat{h}_{k}}}
	&\leq \frac{\sqtwonorm{y_{1}  - y^{*}} - \EE{\sqtwonorm{y_{n+1} - y^{*}}}}{2\eta} + 
\frac{\eta}{2}\sum_{k=1}^{n}G^2\\
	&\leq \frac{\sqtwonorm{y_{1}-y^{*}}}{2\eta} + \frac{\eta nG^2}{2}.
	\end{align*}
\end{proof}

The next result is a similar regret analysis for mirror descent with the relative entropy as its distance generating 
function. Once again, this result is standard, and we refer the interested reader to \citet{NY83,CBLu06:book,Ora19} for 
more details. For the analysis, we recall that $\DD$ denotes the relative entropy (or Kullback--Leibler divergence), 
defined for any $p,q\in\Delta_\A$ as $\DDKL{p}{q}=\sum_{a}p(a)\log\frac{p(a)}{q(a)}$, and that, for any two policies 
$\pi,\pi'$, we define the conditional entropy\footnote{Technically speaking, this quantity is the conditional entropy 
between the occupancy measures $\mu^{\pi}$ and $\mu^{\pi'}$. We will continue to use this relatively imprecise 
terminology to keep our notation light, and we refer to \citet{NJG17} and \citet{BasSerrano2021} for more details.} 
$\HHKL{\pi}{\pi'} \doteq \sum_{x\in\X}\nu^{\pi}(x)\DDKL{\pi(\cdot|x)}{\pi'(\cdot|x)}$.

\begin{lemma}[Mirror Descent]\label{lem:aux-mirror}
	Let $q_t,\dots,q_T$ be a sequence of functions from $\X\times\A$ to $\Reals$ so that $\norm{q_t}_{\infty}\le D_q$ for $t=1,\dots,T$.
	Given an initial policy $\pi_1$ and a learning rate $\alpha>0$, define the sequence of policies $\pi_2,\dots,\pi_{T+1}$ such that, for $t=1,\dots,T$:
	\begin{equation*}
		\pi_{t+1}(a|x) \propto \pi_te^{\alpha q_t(x,a)}.
	\end{equation*}
	Then, for any comparator policy $\pi^*$:
	\begin{align*}
	\sum_{t=1}^{T}\sum_{x\in\X}\nu^{\pi^*}(x)\iprod{\pi^{*}(\cdot|x) - \pi_{t}(\cdot|x)}{q_{t}(x,\cdot)}\leq \frac{\HHKL{\pi^*}{\pi_1}}{\alpha} + \frac{\alpha TD_{q}^2}{2}.
	\end{align*}
\end{lemma}
\begin{proof}
	We begin by studying the relative entropy between $\pi^{*}(\cdot|x)$ and iterates $\pi_{t}(\cdot|x),\pi_{t+1}(\cdot|x)$ for any $x\in\X$:
	\begin{align*}
	\DDKL{\pi^{*}(\cdot|x)}{\pi_{t+1}(\cdot|x)}
	&= \DDKL{\pi^{*}(\cdot|x)}{\pi_{t}(\cdot|x)} - \sum_{a\in\A}\pi^{*}(a|x)\log\frac{\pi_{t+1}(a|x)}{\pi_{t}(a|x)}\\
	&= \DDKL{\pi^{*}(\cdot|x)}{\pi_{t}(\cdot|x)} - \sum_{a\in\A}\pi^{*}(a|x)\log\frac{e^{\alpha q_{t}(x,a)}}{\sum_{a'\in\A}\pi_{t}(a'|x)e^{\alpha q_{t}(x,a')}}\\
	&= \DDKL{\pi^{*}(\cdot|x)}{\pi_{t}(\cdot|x)} - \alpha\iprod{\pi^{*}(\cdot|x)}{q_{t}(x,\cdot)} + \log\sum_{a\in\A}\pi_{t}(a|x)e^{\alpha q_{t}(x,a)}\\
	&= \DDKL{\pi^{*}(\cdot|x)}{\pi_{t}(\cdot|x)} - \alpha\iprod{\pi^{*}(\cdot|x) - \pi_{t}(\cdot|x)}{q_{t}(x,\cdot)}\\
	&\quad+ \log\sum_{a\in\A}\pi_{t}(a|x)e^{\alpha q_{t}(x,a)} - \alpha\sum_{a\in\A}\pi_{t}(a|x)q_{t}(x,a)\\
	&\leq \DDKL{\pi^{*}(\cdot|x)}{\pi_{t}(\cdot|x)} - \alpha\iprod{\pi^{*}(\cdot|x) - \pi_{t}(\cdot|x)}{q_{t}(x,\cdot)}
	+\frac{\alpha^{2}\infnorm{q_{t}(x,\cdot)}^{2}}{2}
	\end{align*}
	where the last inequality follows from Hoeffding's lemma (cf.~Lemma A.1 in~\citealp{CBLu06:book}). Next, we 
rearrange the above equation, sum over $t=1,\cdots,T$, evaluate the resulting telescoping sum and upper-bound negative 
terms by zero to obtain
	\begin{align*}
	\sum_{t=1}^{T}\iprod{\pi^{*}(\cdot|x) - \pi_{t}(\cdot|x)}{q_{t}(x,\cdot)}
	&\leq \frac{\DDKL{\pi^{*}(\cdot|x)}{\pi_{1}(\cdot|x)}}{\alpha}
	+ \frac{\alpha\infnorm{q_{t}(x,\cdot)}^{2}}{2}.
	\end{align*}
	Finally, using that $\|q_{t}\|_{\infty}\leq D_{q}$ and taking an expectation with respect to $x\sim\nu^{\pi^*}$ 
concludes the proof.
\end{proof}

\newpage
\section{Detailed Computations for Comparing Coverage Ratios}\label{app:ratios}
For ease of comparison, we just consider discounted linear MDPs (Definition \ref{def:linMDP}).

\begin{definition}\label{def:cr}
Recall the following definitions of coverage ratio given by different authors in the offline RL literature:
\begin{enumerate}
	\item $C_{\varphi,c}(\pi^*;\pi_B) = \EEs{\feat[X,A]}{X,A\sim\mu^*}^\top\Vm^{-2c}\EEs{\feat[X,A]}{X,A\sim\mu^*}$\hfill (Ours)
	\item $C^{\diamond}(\pi^*;\pi_B) = \EEs{\fvec(X,A)\transpose\Vm^{-1}\fvec(X,A)}{X,A\sim\mu^*}$\hfill (e.g., \citet{jin2021pessimism})
	\item $C^{\dagger}(\pi^*;\pi_B) = \sup_{y\in\Reals^d}\frac{y\transpose\EEs{\fvec(X,A)\fvec(X,A)\transpose}{X,A\sim\mu^*}y}{y\transpose\EEs{\fvec(X,A)\fvec(X,A)\transpose}{X,A\sim\mu_B}y}$\hfill (e.g., \citet{uehara2022pessimistic})
	\item $C_{\mathcal{F},\pi}(\pi^*;\pi_B) = \max_{f\in\mathcal{F}}\frac{\norm{f-\mathcal{T}^\pi f}_{\mu^*}^2}{\norm{f-\mathcal{T}^\pi f}_{\mu_B}^2}$\hfill (e.g., \citet{Xie21}),
\end{enumerate}	
where $c\in\{1,2\}$, $\Vm=\EEs{\fvec(X,A)\fvec(X,A)\transpose}{X,A\sim\mu_B}$ (assumed invertible), $\mathcal{F}\subseteq \Reals^{\X\times\A}$, and $\mathcal{T}^\pi:\mathcal{F}\to\Reals$ defined as $(\mathcal{T}^\pi f)(x,a)=r(x,a)+\gamma \sum_{x',a'}p(x'\mid x,a)\pi(a'\mid x')f(x',a')$ is the Bellman operator associated to policy $\pi$.
\end{definition}

The following is a generalization of the low-variance property from Section~\ref{sec:discussion}.
\begin{proposition}\label{prop:cr1}
	Let $\Vars{Z}{} = \mathbb{E}[\norm{Z-\EE{Z}}^2]$ for a random vector $Z$. Then
	\begin{equation*}
		C_{\varphi,c}(\pi^*;\pi_B) = \EEs{\fvec(X,A)\transpose\Vm^{-2c}\fvec(X,A)}{X,A\sim\mu^*} - \Vars{\Vm^{-c}\fvec(X,A)}{X,A\sim\mu^*}.
	\end{equation*}
\end{proposition}
\begin{proof}
	We just rewrite $C_{\varphi,c}$ from Definition~\ref{def:cr} as
	\begin{equation*}
		C_{\varphi,c}(\pi^*;\pi_B) = \norm{\EEs{\Vm^{-c}\feat[X,A]}{X,A\sim\mu^*}}^2.
	\end{equation*}
	The result follows from the elementary property of variance $\Vars{Z}{}=\mathbb{E}[\norm{Z}^2] - \norm{\mathbb{E}[Z]}^2$.
\end{proof}

\begin{proposition}\label{prop:cr2}
	$C^{\dagger}(\pi^*;\pi_B)\le C^{\diamond}(\pi^*;\pi_B)\le dC^{\dagger}(\pi^*;\pi_B)$.
\end{proposition}
\begin{proof}
	Let $(X^*,A^*)\sim \mu^*$ and $\bm{M}=\EE{\feat[X^*,A^*]\feat[X^*,A^*]}$. First, we rewrite $C^{\diamond}$ as
	\begin{align}
	C^{\diamond}(\pi^*;\pi_B) &= \EE{\feat[X^*,A^*]\transpose\Vm^{-1}\feat[X^*,A^*]}\nonumber\\
	&=\EE{\Tr(\feat[X^*,A^*]\transpose\Vm^{-1}\feat[X^*,A^*])}\nonumber\\
	&=\EE{\Tr(\feat[X^*,A^*]\feat[X^*,A^*]\transpose\Vm^{-1})}\\
	&=\Tr(\bm{M}\Vm^{-1})\\
	&=\Tr(\Vm^{-1/2}\bm{M}\Vm^{-1/2}),
	\end{align}
	where we have used the cyclic property of the trace (twice) and linearity of trace and expectation. Note that, since $\Vm$ is positive definite, it admits a unique positive definite matrix $\Vm^{1/2}$ such that $\Vm=\Vm^{1/2}\Vm^{1/2}$. We rewrite $C^\dagger$ in a similar fashion
	\begin{align}
	C^\dagger(\pi^*;\pi_B) &= 
	\sup_{y\in\Reals^d}\frac{y\transpose \bm{M} y}{y\transpose\Vm y} \nonumber\\
	&=	\sup_{z\in\Reals^d}\frac{z\transpose \Vm^{-1/2}\bm{M} \Vm^{-1/2} z}{z\transpose z}\\
	&=\lambdamax( \Vm^{-1/2}\bm{M} \Vm^{-1/2}),
	\end{align}
	where $\lambdamax$ denotes the maximum eigenvalue of a matrix. We have used the fact that both $\bm{M}$ and $\Vm$ are positive definite and the min-max theorem. Since the quadratic form $\Vm^{-1/2}\bm{M}\Vm^{-1/2}$ is also positive definite, and the trace is the sum of the (positive) eigenvalues, we get the desired result.
\end{proof}

\begin{proposition}[cf. the proof of Theorem 3.2 from~\citep{Xie21}]\label{prop:cr3}
	Let $\mathcal{F}=\{f_{\tvec}:(x,a)\mapsto \langle \feat,\tvec\rangle\mid \tvec\in\Theta\subseteq\Reals^d\}$ where $\varphi$ is the feature map of the linear MDP. Then
	\begin{equation*}
		C_{\mathcal{F},\pi}(\pi^*;\pi_B) \le C^\dagger(\pi^*;\pi_B),
	\end{equation*}
	with equality if $\Theta=\Reals^d$.
\end{proposition}
\begin{proof}
	Fix any policy $\pi$ and let $\mathcal{T}=\mathcal{T}^\pi$. By linear Bellman completeness of linear MDPs~\citep{Jin2020}, $\mathcal{T}f\in\mathcal{F}$ for any $f\in\mathcal{F}$. For $f_{\tvec}:(x,a)\mapsto\ip{\feat[x,a]}{\tvec}$, let $\mathcal{T}\tvec\in\Theta$ be defined so that $\mathcal{T}f_{\tvec}:(x,a)\mapsto\ip{\feat[x,a]}{\mathcal{T}\tvec}$. Then
	\begin{align}
	C_{\mathcal{F},\pi}(\pi^*;\pi_B) 
	&=\max_{f\in\mathcal{F}}\frac{\EEs{\left(f(X,A)-\mathcal{T} f(X,A)\right)^2}{X,A\sim\mu^*}}{\EEs{\left(f(X,A)-\mathcal{T} f(X,A)\right)^2}{X,A\sim\mu_B}} \\
	&\le\max_{\tvec\in\Reals^d}\frac{\EEs{\ip{\feat[X,A]}{\tvec-\mathcal{T}\tvec}^2}{X,A\sim\mu^*}}{\EEs{\ip{\feat[X,A]}{\tvec-\mathcal{T}\tvec}^2}{X,A\sim\mu_B}}\label{ineq} \\
	&= \max_{y\in\Reals^d}\frac{\EEs{\ip{\feat[X,A]}{y}^2}{X,A\sim\mu^*}}{\EEs{\ip{\feat[X,A]}{y}^2}{X,A\sim\mu_B}} \\
	&= \max_{y\in\Reals^d}\frac{y\transpose\EEs{\feat[X,A]\feat[X,A]\transpose}{X,A\sim\mu^*}y}{y\transpose\EEs{\feat[X,A]\feat[X,A]\transpose}{X,A\sim\mu_B}y},
	\end{align}
	where the inequality in Equation~\eqref{ineq} holds with equality if $\Theta=\Reals^d$.
\end{proof}

\end{document}